\documentclass{article}

\usepackage{microtype}
\usepackage{graphicx}
\usepackage{booktabs} 

\usepackage{hyperref}

\usepackage[accepted]{icml2023}

\usepackage{amsmath}
\usepackage{amssymb}
\usepackage{mathtools}
\usepackage{amsthm}
\usepackage{multirow}

\usepackage{xspace}
\definecolor{mine}{RGB}{255, 230, 163}

\usepackage[capitalize,noabbrev]{cleveref}

\newcommand{\E}{\mathbb{E}}
\newcommand{\J}{\mathcal{J}}

\theoremstyle{plain}
\newtheorem{theorem}{Theorem}[section]

\newtheorem{lemma}[theorem]{Lemma}

\theoremstyle{definition}

\theoremstyle{remark}

\usepackage[textsize=tiny]{todonotes}

\DeclareMathOperator*{\argmin}{arg\,min}

\usepackage{xspace}
\usepackage{subfig}
\usepackage{bm}
\usepackage{wrapfig}
\usepackage{tikz}
\usetikzlibrary{bayesnet}
\usetikzlibrary{arrows}
\usetikzlibrary{backgrounds}
\usepackage{dsfont}

\icmltitlerunning{Understanding Hindsight Goal Relabeling}

\begin{document}

\twocolumn[
\icmltitle{Understanding Hindsight Goal Relabeling \\
           from a Divergence Minimization Perspective}



\icmlsetsymbol{equal}{*}

\begin{icmlauthorlist}
\icmlauthor{Lunjun Zhang}{uoft,vector,waabi}
\icmlauthor{Bradly C. Stadie}{northwestern}
\end{icmlauthorlist}

\icmlaffiliation{uoft}{Department of Computer Science, University of Toronto}
\icmlaffiliation{vector}{Vector Institute}
\icmlaffiliation{waabi}{Waabi}
\icmlaffiliation{northwestern}{Department of Statistics, Northwestern University}

\icmlcorrespondingauthor{Lunjun Zhang}{lunjun@cs.toronto.edu}

\icmlkeywords{reinforcement learning, imitation learning, goal reaching, reward design}

\vskip 0.3in
]



\printAffiliationsAndNotice{}  

\begin{abstract}
Hindsight goal relabeling has become a foundational technique in multi-goal reinforcement learning (RL). The essential idea is that any trajectory can be seen as a sub-optimal demonstration for reaching its final state. Intuitively, learning from those arbitrary demonstrations can be seen as a form of imitation learning (IL). However, the connection between hindsight goal relabeling and imitation learning is not well understood. In this paper, we propose a novel framework to understand hindsight goal relabeling from a divergence minimization perspective. Recasting the goal reaching problem in the IL framework not only allows us to derive several existing methods from first principles, but also provides us with the tools from IL to improve goal reaching algorithms. Experimentally, we find that under hindsight relabeling, $Q$-learning outperforms behaviour cloning (BC). Yet, a vanilla combination of both hurts performance. Concretely, we see that the BC loss only helps when selectively applied to actions that get the agent closer to the goal according to the $Q$-function. Our framework also explains the puzzling phenomenon wherein a reward of $\{-1, 0\}$ results in significantly better performance than a $\{0, 1\}$ reward for goal reaching.

\end{abstract}

\section{Introduction}

Goal reaching is an essential aspect of intelligence in sequential decision making. Unlike the conventional formulation of reinforcement learning (RL), which aims to encode all desired behaviors into a single scalar reward function that is amenable to learning \citep{reward-is-enough}, goal reaching formulates the problem of RL as applying a sequence of actions to \textit{rearrange} the environment into a desired state \citep{rearrangement}. Goal-reaching is a highly flexible formulation. For instance, we can design the goal-space to capture salient information about specific factors of variations \citep{multi-goal-rl}; we can use natural language instructions to define more abstract goals \citep{lynch2020language, say-can}; we can encourage exploration by prioritizing previously unseen goals \citep{skew-fit, discern, mega}; and we can even use self-supervised procedures to naturally learn goal-reaching policies without reward engineering \citep{tdm, rig, world-model-graph, openai2021asymmetric, actionablemodels}. 

Imitation learning (IL) aims to recover an expert policy from a set of expert demonstrations. The simplest imitation learning algorithm is Behaviour Cloning (BC), which directly uses the given demonstrations to supervise the policy actions conditioned on the states visited by the expert \citep{alvinn}. An alternative approach, inverse reinforcement learning (IRL), first learns a reward function from the demonstrations, and then runs online RL on the learned reward to extract the policy \citep{apprenticeship}. A modern example of IRL is generative adversarial imitation learning (GAIL) \citep{gail}, which learns a discriminator \citep{gan} as the reward while jointly running policy gradient algorithms. It has been shown that the majority of imitation learning methods can be unified under the divergence minimization framework \citep{il-divergence, f-gail, ke2021imitation}. 

\begin{figure*}[h]
    \centering
    \includegraphics[width=0.7\textwidth]{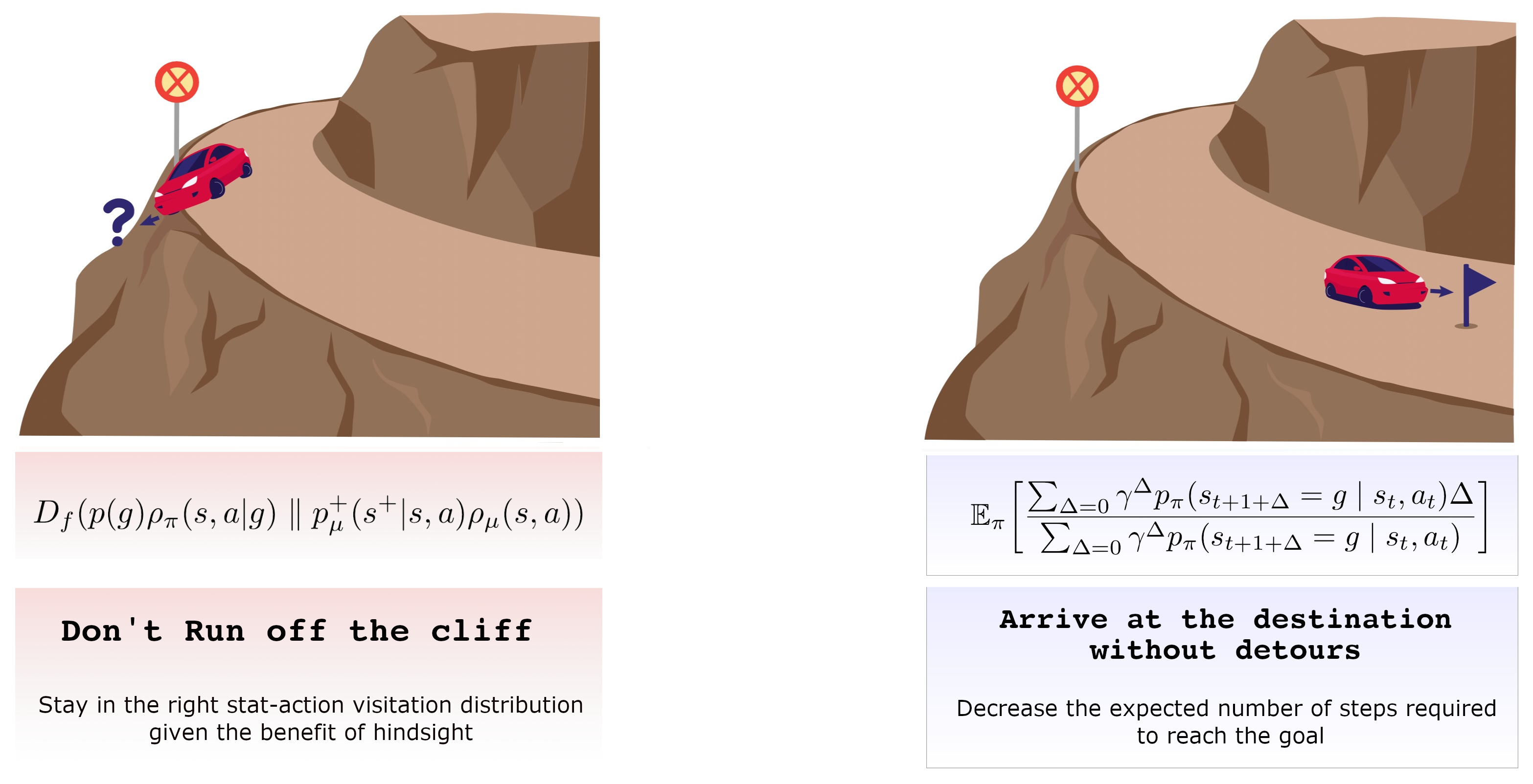}
    \caption{\footnotesize Our proposed framework for goal reaching that elucidates how rewards should be designed in multi-goal RL.}
    \label{fig:objective-anology}
    \vspace{-3mm}
\end{figure*}

Hindsight goal relabeling is a technique proposed in \cite{her} to improve the sample efficiency of goal reaching and has since been widely used \citep{gcsl, latent-plan-from-play, actionablemodels}. A goal-conditioned policy typically sets a behavioral goal and then tries to reach it during sampling, but most attempts will likely fail. Nevertheless, any trajectory has successfully reached all the states it visits along the way. Therefore, the agent may pretend post hoc that whatever states it reaches are the intended goals of that trajectory, and learn from its own success. Intuitively, hindsight relabeling creates self-generated expert demonstrations, which the policy then imitates. Can we mathematically describe hindsight relabeling and goal reaching as an imitation learning process from a divergence minimization perspective? 

In this paper, we propose a novel framework that describes hindsight goal relabeling in the language of divergence minimization, further bridging the two paradigms of goal reaching and imitation learning. Compared to prior attempts to formulate a theoretical framework for goal reaching \citep{hindsight-em, rewriting-history-with-inverse-rl, learning-successor-states, c-learning, outcome-driven}, our work pays substantially more attention to reward design in goal reaching and its connection to inverse RL, and having a probabilistically well-defined objective in the case of continuous state space. In addition, our framework can re-derive several existing methods from first principles, unlike prior frameworks that are incompatible with existing methods and thus unable to explain the success of hindsight-relabeling based goal reaching. For experiments, we find that despite a resurgence of interests in BC-based goal reaching \citep{gcsl, hbc, latent-plan-from-play, bc-zero}, multi-goal $Q$-learning can still outperform BC, and adding BC loss in multi-goal $Q$-learning hurts performance due to the sub-optimality of self-generated demonstrations. Indeed, BC only helps when selectively applied to actions that get the agent closer to the goal according to the $Q$-function. Finally, our framework also provides an interesting explanation for why the rewards of $-1$ and $0$ experimentally outperforms the rewards of $0$ and $1$ in multi-goal $Q$-learning. 

\section{Goal Reaching and Imitation Learning}

\begin{figure*}[!htb]
\captionsetup{justification=centering}
\begin{minipage}{0.33\textwidth}
    \centering
    \begin{tikzpicture}[scale = 0.4]
     \node[obs] (s) {$s$}; %
     \node[obs, right=of s] (a) {$a$}; %
     \node[obs, above=of a] (g) {$s^{+}$}; %
     \node[latent, left=of s] (env) {$\xi$};
     \node[latent, above=of env] (mu) {$\mu$}; %
     \plate [inner sep=.15cm, yshift=.15cm] {plate1} {(s)(a)(g)(mu)(env)} {}; %
     \edge {mu} {s, a}
     \edge {mu} {g}
     \edge {s, a} {g}
     \edge {s} {a}
     \edge {env} {s, g}
    \end{tikzpicture} 
    \caption{ $\rho_{\mu}(s, a) p^{+}_{\mu} (s^{+} \mid s, a)$ \\ hindsight-relabeled distribution}
    \label{target-distribution}
\end{minipage}
    \begin{minipage}{0.33\textwidth}
    \centering
    \tikz{
     \node[obs] (s) {$s$}; %
     \node[obs, right=of s] (a) {$a$}; %
     \node[obs, above=of a] (g) {$g$}; %
     \node[obs, above=of s] (pi) {$\pi$}; %
     \node[latent, left=of s] (env) {$\xi$};
     \node[latent, left=of pi] (mu) {$\mu$};
     \plate [inner sep=.15cm, yshift=.15cm] {plate1} {(s)(a)(g)(pi)(env)(mu)} {}; %
     \edge {pi} {a}
     \edge {g} {pi}
     \edge {s} {a}
     \edge {env} {s, g}
     \edge {mu} {s}
    }
    \caption{$\rho^{+}(g) \rho_{\mu}(s) \pi(a\mid s, g)$ \\
    (used in HBC / GCSL)
    }
    \label{hbc-distribution}
\end{minipage}
\begin{minipage}{0.33\textwidth}
\centering
\tikz{
 \node[obs] (s) {$s$}; %
 \node[obs, right=of s] (a) {$a$}; %
 \node[obs, above=of a] (g) {$g$}; %
 \node[obs, above=of s] (pi) {$\pi$}; %
 \node[latent, left=of s] (env) {$\xi$}; 
 \node[latent, above=of env] (mu) {$\mu$}; %
 \plate [inner sep=.15cm, yshift=.15cm] {plate1} {(s)(a)(g)(pi)(env)} {}; %
 \edge {pi} {s}
 \edge {pi} {a}
 \edge {g} {pi}
 \edge {s} {a}
 \edge {env} {s, g}
}
\caption{$\rho^{+}(g) \rho_{\pi}(s, a\mid g)$ \\ 
(used in HER / HDM)
}
 \label{her-distribution}
\end{minipage}
\end{figure*}

\subsection{Goal-Conditioned Reinforcement Learning (GCRL)}
We first review the basics of RL. A Markov Decision Process (MDP) is typically parameterized by ($\mathcal{S}$, $\mathcal{A}$, $\rho^{0}$, $p$, $r$): a state space $\mathcal{S}$, an action space $\mathcal{A}$, an initial state distribution $\rho^{0}(s)$, a dynamics function $p(s'\mid s,a)$ which defines the transition probability, and a reward function $r(s,a)$. A policy function $\mu(a\mid s)$ defines a probability distribution $\mu: \mathcal{S} \times \mathcal{A} \rightarrow \mathbb{R}^{+}$. For an infinite-horizon MDP, given the policy $\mu$, and the state distribution at step $t$ (starting from $\rho^{0}$ at $t=0$), the state distribution at step $t+1$ is given by:
\begin{align}
    \rho^{t+1}_{\mu} (s') = \int_{\mathcal{S}\times\mathcal{A}} p(s' \mid s, a) \mu(a \mid s) \rho^{t}_{\mu} (s) \mathrm{d} s \mathrm{d} a \label{state-transition}
\end{align}
The \textit{state visitation distribution} sums over all timesteps via a geometric distribution $\text{Geom} (\gamma)$:
\begin{align}
    \rho_{\mu}(s) = (1 - \gamma) \cdot \sum_{t=0}^{\infty} \gamma^{t}\cdot \rho^{t}_{\mu}(s) \label{state-distribution-aggregate}
\end{align}
However, the trajectory sampling process does not happen in this discounted manner, so the discount factor $\gamma \in (0,1)$ is often absorbed into the cumulative return instead \citep{deterministic-policy-gradient}:
\begin{equation}
\begin{aligned}
    \J(\mu) &= \dfrac{1}{1-\gamma} \int_{\mathcal{S}} \rho_{\mu}(s) \int_{\mathcal{A}} \mu(a\mid s)r(s,a) \mathrm{d}a \mathrm{d}s \\
    & = \E_{ \substack{\rho^{0}(s_{0}) \mu(a_{0}\mid s_{0}) \\ p(s_{1}\mid s_{0}, a_{0}) \mu(a_{1}\mid s_{1}) \cdots } } \Bigg[\sum_{t=0}^{\infty} \gamma^{t} r(s_t, a_t)\Bigg]
\end{aligned}
\end{equation}
From \ref{state-transition} and \ref{state-distribution-aggregate}, we also see that the future state distribution $p^{+}(s^{+}\mid s,a)$ of policy $\mu$, defined as a geometric sum of state distribution at all future timesteps given current state and action, is given by the following recursive relationship \citep{c-learning, gamma-models}:
\begin{equation}
\begin{aligned}
    &p^{+}_{\mu} (s^{+} \mid s, a) \\
    = &(1-\gamma) p(s^{+} \mid s, a) + \gamma \E_{ \substack{p(s' \mid s, a) \\ \mu(a' \mid s')} } [p^{+}_{\mu} (s^{+} \mid s', a')] \label{eq:future-state-distribution}
\end{aligned}
\end{equation}

In multi-goal RL, an MDP is augmented with a goal space $\mathcal{G}$, and we learn a goal-conditioned policy $\pi: \mathcal{S}\times \mathcal{G} \times \mathcal{A} \rightarrow \mathbb{R}^{+}$. Hindsight Experience Replay (HER) \citep{her} gives the agent a reward of $0$ when the goal is reached and $-1$ otherwise, typically determined via an epsilon ball,
\begin{align}
    r(s,a,s',s^{+}) = - \mathds{1} (\lVert s'-s^{+} \rVert > \epsilon) \label{eq:her-reward}
\end{align}
and uses \textit{hindsight goal relabeling} to increase learning efficiency of goal-conditioned $Q$-learning by replacing the initial \textit{behavioral goals} with \textit{achieved goals} (future states within the same trajectory). 

\subsection{Imitation Learning (IL) as Divergence Minimization}

We review the concept of $f$-divergence between two probability distributions $P$ and $Q$ and its variational bound:
\begin{equation}
\begin{aligned}
    D_{f}(P \parallel Q) &= \int_{X} q(x) f\bigg( \dfrac{p(x)}{q(x)} \bigg) \mathrm{d} x \\
    &\geq \sup_{T\in \mathcal{T}} \E_{x\sim P}[T(x)] - \E_{x\sim Q}[f^{*}(T(x))] \label{eq:f-divergence}
\end{aligned}
\end{equation}
where $f$ is a convex function such that $f(1)=0$, $f^{*}$ is the convex conjugate of $f$, and $T$ is an arbitrary class of functions $T: X \rightarrow \mathbb{R}$. This variational bound was originally derived in \citep{convex-risk} and was popularized by GAN \citep{gan, fgan} and subsequently by imitation learning \citep{gail, airl, finn2016connection, il-divergence}. The equality holds true under mild conditions \citep{convex-risk}, and the optimal $T$ is given by $T^{*}(x) = f'( {p(x)} / {q(x)})$. The canonical formulation of imitation learning follows \citep{gail, il-divergence}, where $\rho^{\text{exp}}(s,a)$ is from the expert:
\begin{equation}
\begin{aligned}
    &\argmin_{\mu} D_{f}(\rho^{\text{exp}}(s,a) \parallel \rho^{\mu}(s,a)) \\
    = &\argmin_{\mu} \max_{T} \E_{\rho^{\text{exp}}} [T(s,a)] - \E_{\rho^{\mu}}[f^{*}(T(s,a))] \label{gail-derivation}
\end{aligned}
\end{equation}
Because of the \textit{policy gradient theorem} \citep{policy-gradient-methods}, the policy $\mu$ needs to optimize the \textit{cumulative} return under its own trajectory distribution $\rho^{\mu}(s,a)$ with the reward being $r(s,a)=f^{*}(T(s,a))$. Under this formulation, Jensen-Shannon divergence leads to GAIL \citep{gail}, reverse KL leads to AIRL \citep{airl}. Note that $f$ can in principle be any convex function (we can satisfy $f(1)=0$ by simply adding a constant).

\section{Graphical Models for Hindsight Goal Relabeling}

Consider an environment $\xi=(\rho^{0}(s), p(s'\mid s,a), p(g))$. From $\xi$, we generate a dataset of trajectories $\mathcal{D}=\{(s_{0}, a_{0}, s_{1}, a_{1}, \cdots)\}$ by first sampling from the initial state distribution $\rho_{0}(s)$, and then executing an unobserved actor policy $\mu(a\mid s)$. After executing this action, transitions occur according to a dynamics model $p(s'\mid s,a)$. We aim to train a goal-conditioned policy $\pi(a\mid s,g)$ from this arbitrary dataset with relabeled future states as goals. $p(g)$ is the behavioral goal distribution assumed to be given \textit{apriori} by the environment. To recast the problem of goal-reaching as imitation learning \eqref{gail-derivation}, we need to set up an $f$-divergence minimization where we define the target (expert) distribution and the policy distribution we want to match. 

In multi-goal RL, the training signal comes from factorizing the joint distribution of state-action-goal differently. For the relabeled target distribution, we assume an unconditioned actor $\mu$ generating a state-action distribution at first, with the goals coming from the future state distribution \eqref{eq:future-state-distribution} conditioned on the given state and action. For the goal-conditioned policy distribution, behavior goals are given \textit{apriori}, and the state-action distribution is generated conditioned on the behavioral goals. Thus, the \textit{target distribution} (see Figure \ref{target-distribution}) for states, actions, and \textit{hindsight goals} is:
\begin{align}
    p_{\mu}(s, a, s^{+}) = \rho_{\mu}(s, a) p^{+}_{\mu} (s^{+} \mid s, a) 
\end{align}
Concretely, the target distribution is given by the unconditioned state-action visitation distribution multiplied by the conditional likelihood of future states. Note that $p_{\mu}^{+}$ is given by equation \eqref{eq:future-state-distribution}, and $\rho_{\mu}(s, a)$ is similar to $\rho^{\text{exp}}(s,a)$ in equation \eqref{gail-derivation}.
In the fashion of behavioral cloning (BC), if we do not care about matching the states, we can write the joint distribution we are trying to match as (see Figure \ref{hbc-distribution}):
\begin{align}
    p_{\pi}^{\text{BC}}(s, a, g) = {p(g)}  \rho_{\mu}(s) \pi(a\mid s, g) 
\end{align}
We can recover the objective of Hindsight Behavior Cloning (HBC) \citep{hbc, rewriting-history-with-inverse-rl} and Goal-Conditioned Supervised Learning (GCSL) \cite{gcsl} via minimizing a KL-divergence (see \ref{sec:hbc-proof}):
\begin{equation}
\begin{aligned}
&\argmin_{\pi} \mathcal{D}_{KL} (p_{\mu}(s, a, s^{+}) \parallel p_{\pi}^{\text{BC}}(s, a, g) ) \\
= &\argmin_{\pi} \E_{ \rho_{\mu}(s, a) p^{+}_{\mu} (s^{+} \mid s, a)}  [ -\log \pi(a \mid s, g) ] \label{hbc-derivation}
\end{aligned}
\end{equation}
In many cases, matching the states is more important than matching state-conditioned actions \citep{dagger, il-divergence}. The joint distribution for states, actions, \textit{behavioral goals} for $\pi$ (see Figure \ref{her-distribution}) is:
\begin{align}
    p_{\pi}(s, a, g) = {p(g)} \rho_{\pi}(s, a\mid g) \label{eq:her-graphical-model}
\end{align}
Resulting in the $f$-divergence objective for $\pi$ to minimize:
\begin{align}
    \mathcal{L}_{IL} = {D_{f} ( p(g)\rho_{\pi}(s, a | g) \parallel p_{\mu}^{+}(s^{+} | s, a) \rho_{\mu}(s, a) )} \label{imitation-learning-term}
\end{align}
After adding this state-matching term, we are still missing an important component of the objective. Divergence minimization encourages the agent to stay in the "right" state-action distribution given the benefit of hindsight. But we still want the policy to actually hit the goal as quickly as possible without detours (see Figure \ref{fig:objective-anology}). We propose that the goal-conditioned policy $\pi$ should therefore minimize:
\begin{align}
    \mathcal{L}_{RL} = \E_{\pi} \bigg[\dfrac{ \sum_{\Delta=0} \gamma^{\Delta} p_{\pi}(s_{t+1+\Delta} = g \mid s_{t}, a_{t}) \Delta }{ \sum_{\Delta=0} \gamma^{\Delta} p_{\pi}(s_{t+1+\Delta} = g \mid s_{t}, a_{t}) } \bigg] \label{eq:goal-reaching-fewer-steps}
\end{align}
Which estimates the expected number of steps for the policy to reach a certain goal. The overall objective for goal-conditioned RL can seen as a combination of $\mathcal{L}_{IL}$ and $ \mathcal{L}_{RL}$.
We will expand the two objectives in more details and discuss how to use $Q$-learning to optimize them.

\section{Bridging Goal Reaching and Imitation Learning}

In this section, we study how to optimize the proposed goal reaching objectives \eqref{imitation-learning-term} and \eqref{eq:goal-reaching-fewer-steps}, with the goal of deriving and understanding the $Q$-learning process and reward design in multi-goal RL. 

\subsection{Divergence Minimization with Goal-Conditioned Q-learning}

\label{section:f-divergence}

This section decomposes \eqref{imitation-learning-term}. We start with the $f$-divergence bound from equations \eqref{eq:f-divergence} and \eqref{gail-derivation}. 
\begin{equation*}
    \begin{aligned}
        &D_{f}( p_{\pi}(s, a, g) \parallel p_{\mu}(s, a, s^{+}) ) \\
        = & \max_{T} \E_{ \substack{p(g) \\ \rho_{\pi}(s, a\mid g) } } [T(s, a, g)] - \E_{ \substack{\rho_{\mu}(s, a) \\ p^{+}_{\mu} (s^{+} \mid s, a)} } [f^{*}(T(s, a, s^{+}))] \label{eq:f-divergence-hdm-init}
    \end{aligned}
\end{equation*}
Now we negate $T$ to get $r(s,a,g) = -T(s,a,g)$, and the divergence minimization problem becomes:
\begin{align*}
    \max_{\pi} \min_{r} \E_{ \substack{\rho_{\mu}(s, a) \\  p^{+}_{\mu} (s^{+} \mid s, a)} } [f^{*}(-r(s, a, s^{+}))] + \E_{ \substack{p(g) \\ \rho_{\pi}(s, a\mid g) } } [r(s, a, g)] \label{eq:gail-goal-conditioned}
\end{align*}
We can interpret $r$ as a GAIL-style \citep{gail} discriminator or reward. However, we aim to derive a discriminator-free learning process that directly trains the $Q$-function corresponding to this reward $Q(s,a,g) = r(s,a,g) + \gamma \cdot \mathcal{P}^{\pi}Q(s,a,g)$ where $\mathcal{P}^{\pi}$ is the \textit{transition operator}: $\mathcal{P}^{\pi}Q(s,a,g) = \E_{p(s'\mid s,a) \pi(a'\mid s', g)} [Q_(s',a',g)]$. Re-writing the previous equation w.r.t $Q$:
\begin{equation}
    \begin{aligned}
        \min_{Q} \E_{ \substack{\rho_{\mu}(s, a) \\ p^{+}_{\mu} (s^{+} \mid s, a)} } [f^{*}(-(Q-\gamma \cdot \mathcal{P}^{\pi}Q)(s, a, s^{+}))] \\
        + \E_{ \substack{p(g) \\ \rho_{\pi}(s, a\mid g) } } [(Q-\gamma \cdot \mathcal{P}^{\pi}Q)(s, a, g)]
    \end{aligned}
\end{equation}
A similar change-of-variable has been explored in the context of offline RL \citep{dualdice, algaedice} and imitation learning \citep{valuedice, OPOLO}, known as the DICE \citep{rl-fenchel-rockafellar} family. A major pain point of DICE-like methods is that they require samples from the initial state distribution $\rho^{0}(s)$ \citep{iq-learn}. The following lemma shows that in the goal-conditioned case, we can use arbitrary offline trajectories to evaluate the expected rewards under goal-conditioned online rollouts:
\begin{lemma}[Online-to-offline transformation for goal reaching]
\label{lemma:online-to-offline}
Given a goal-conditioned policy\\ $\pi(a\mid s,g)$, its corresponding $Q$-function $Q^{\pi}(s,a,g)$, and arbitrary state-action visitation distribution $\rho_{\mu}(s,a)$ of another policy $\mu(a\mid s)$, the expected temporal difference for online rollouts under $\pi$ is: 
\begin{equation*}
\begin{aligned}
& \E_{ p(g) {\color{blue}\rho_{\pi}(s, a\mid g)} } [(Q^{\pi}-\gamma \cdot \mathcal{P}^{\pi}Q^{\pi})({\color{blue} s}, {\color{blue} a}, g)] \\
= & \E_{ p(g) {\color{red}{\rho_{\mu}(s,a)}} {\color{blue} \pi(\tilde{a}\mid s,g)}  } [Q^{\pi}({\color{red} s}, {\color{blue} \tilde{a}}, g) - \gamma \cdot \mathcal{P}^{\pi}Q^{\pi}({\color{red} s}, {\color{red} a}, g) ]
\end{aligned}
\end{equation*}
\end{lemma}

Using Lemma \ref{lemma:online-to-offline}, the imitation objective now becomes:
\begin{equation}
    \begin{aligned}
        \max_{\pi} & \min_{Q} \E_{ \substack{\rho_{\mu}(s, a) p^{+}_{\mu} (s^{+} \mid s, a) }} \Big\{ f^{*}(-(Q- \gamma \mathcal{P}^{\pi}Q)(s,a,s^{+})) \\
        + & \E_{p(g) \pi(\tilde{a} \mid s, g)} [Q(s,\tilde{a},g) - \gamma \mathcal{P}^{\pi} Q(s,a,g)] \Big\} \label{eq:hdm-divergence-minimization}
    \end{aligned}
\end{equation}

where $f^{*}$ is the convex conjugate of $f$ in $f$-divergence. We can pick almost any convex function as $f^{*}$ as long as $((f^{*})^{*})(1) = 0$. In summary, we have derived a way to minimize the imitation term $\mathcal{L}_{IL} = \mathcal{D}_{f} (p_{\mu}(s, a, s^{+}) \parallel p_{\pi}(s, a, g) )$ directly using goal-conditioned $Q$-learning.

\subsection{Learning to Reach Goals with Fewer Steps}
\label{sec:reach-goals-fewer-steps}
A goal-conditioned agent should try to reach the desired goal using as few steps as possible. Consequently, we may want to learn the expected number of steps to reach another goal state from the current state, and then uses the policy $\pi$ to minimize the expected number of steps. 

The challenge, however, is that defining this value rigorously for continuous state space is a difficult task. With continuous state space, the delta function $\mathds{1}[s' = g]$ is always zero, and the dynamics function $p(s'\mid s,a)$ has a value in $\mathbb{R}^{+}$ rather than $[0,1]$; we can only estimate the likelihood $p(s'=g\mid s,a)$ whose range is $\mathbb{R}^{+}$ and thus cannot be normalized to be either in $\{0,1\}$ or within $[0,1]$ per step for the purpose of step counting. To tackle this issue, we propose the following definition to measure the expected number of steps (negated) from one state to another:
\begin{align}
    Q(s_{t}, a, g) &= - \dfrac{ \sum_{\Delta=0} \gamma^{\Delta} p_{\pi}(s_{t+1+\Delta} = g \mid s_{t}, a) \Delta }{ \sum_{\Delta=0} \gamma^{\Delta} p_{\pi}(s_{t+1+\Delta} = g \mid s_{t}, a) } \label{eq:goal-reaching-step-count}
\end{align}
which normalizes the likelihood of reaching the goal at different $\Delta$ numbers of steps away as discrete probabilities over the step count, with an additional $\gamma$ factor. Maximizing this $Q$ function means minimizing \eqref{eq:goal-reaching-fewer-steps}.
\begin{lemma}[Recursive estimate of goal reaching step count]
Under mild assumptions, the step count definition \eqref{eq:goal-reaching-step-count} has the following property:
\[ 
    Q_{t} = \E_{p(s_{t+1}\mid s_t, a_t)}[(1-\dfrac{\beta_{t} }{ \beta_{t} + \alpha_{t}^{(1)} } ) ( -1 + Q_{t+1} )]
\]
with $Q_{t} = Q(s_{t}, a, g)$, $Q_{t+1} = \E_{\pi(s_{t+1},g)}[Q(s_{t+1}, a', g)]$, $\alpha_{t}^{(i)} = \sum_{\Delta=i} \gamma^{\Delta} p_{\pi}(s_{t+1+\Delta} = g \mid s_{t}, a)$, and  $\beta_{t} = p(s_{t+1}=g \mid s_t, a)$.
\label{lemma:goal-reaching-step-count}
\end{lemma}
The first component $[1 - {\beta_{t} }/{ (\beta_{t} + \alpha_{t}^{(1)}) }]$ is classifying whether the goal can be reached in the immediate next step by comparing the two density ratios; the second component $(-1 + Q_{t+1})$ \textit{shifts} the distance estimate by another step in the case of the goal not being reached at $t+1$. Rather than directly estimating the density ratio, we may adopt a sampling based approach: when sampling from $\beta_{t}$, $Q_{t}$ should be $0$, meaning that $r=0$ (onward); when sampling from $\alpha_{t}^{(1)}$, $r=-1$. Consequently, the definition in \eqref{eq:goal-reaching-step-count} reveals that the reward defined in \eqref{eq:her-reward} forms a well-defined objective even for continuous state space and as $\epsilon\rightarrow 0$.

\subsection{To Reach Goals with Fewer Steps \emph{is} to Imitate}

We now discuss the relationship between the RL objective in \eqref{eq:goal-reaching-fewer-steps} and  \eqref{eq:goal-reaching-step-count} and the objective to imitate in \eqref{imitation-learning-term} and \eqref{eq:hdm-divergence-minimization}. 
\begin{lemma}[Understanding Hindsight Experience Replay]
\label{lemma:her-as-a-special-case}
Multi-goal $Q$-learning with HER reward \eqref{eq:her-reward} with $\{-1, 0\}$ is a special case of minimizing the following objective:
\[
    \E_{ \substack{\rho_{\mu}(s, a) \\ p(s'\mid s,a) \\ p^{+}_{\mu} (s^{+} \mid s, a) }} [f^{*}(-(Q- \gamma \mathcal{P}^{\pi}Q)(s,a,s^{+})) - \beta Q(s,a,s')]
\]
with $\beta=(1-\gamma)$ and the convex function $f^{*}$ chosen to be:
\[
f^{*}(x) = (x-1)^{2} / 2 + 3/2
\]
\end{lemma}
While the step counting interpretation in \eqref{eq:goal-reaching-step-count} already provides valid meaning to the HER rewards, re-writing the objective in this way reveals why this optimization process is also doing imitation learning. Its first term is the same as the one in divergence minimization objective \eqref{eq:hdm-divergence-minimization}. For the second term, rather than pushing \textit{down} on the $Q$ values under the current policy $\pi$ in \eqref{eq:hdm-divergence-minimization}, which would incorrectly assume the optimality of self-generated demonstrations and hinder exploration, it pushes \textit{up} on the transition tuple $(s,a,s')$, where the action $a$ is guaranteed to be optimal for reaching the immediate next state $s'$. As a result, this objective encourages the policy to imitate the best in self-generated demonstrations without restraining exploration.

\subsection{Can Hindsight BC Facilitate Q-learning?}
\label{sec:hindsight-bc-helps}

Now that we have shown that the $Q$-learning process in HER is also doing a special form of imitation, one may wonder whether BC \eqref{hbc-derivation} has any additional role to play. We hypothesize that, for a softmax policy under discrete action space, hindsight BC might facilitate the $Q$-learning process. We have reasons to hypothesize so: under discrete action space, BC becomes a cross entropy loss and has been repeatedly shown to work very well \citep{one-shot-imitation, latent-plan-from-play, bc-zero, bahvior-transformer, robotics-transformer}. Even large language models like GPT3 \citep{gpt3} can be seen as a softmax policy on discrete action space trained with BC. The $Q$-values can be seen as the logits of a softmax policy \citep{equivalence-pg-sql}. The problem is that hindsight relabeled actions $(s,a,s^{+})$ are often suboptimal beyond a single step of relabeling $(s,a,s')$. Utilizing the interpretation of $Q$-values as negated step count in \eqref{eq:goal-reaching-step-count}, we can imitate only those actions that move the agent closer to a goal according to its own estimates, by defining a threshold $w(s,a,s',g)$ for imitation:
\begin{align}
    w = \mathds{1} (Q(s,a,g) - \max_{a'}Q(s',a',g) < \log \gamma_{\text{hdm}})
\end{align}
and the additional objective on top of $Q$-learning with $\{-1, 0\}$ rewards (in \ref{lemma:her-as-a-special-case}) becomes $\mathcal{L}_{\text{hdm}}$:
\begin{align}
    \E_{ \substack{ \rho_{\mu}(s, a, s', g) } } [ - w(s,a,s',g) \cdot \log \text{softmax} Q (s,a,g) ] \label{eq:hdm-final-loss}
\end{align}
which we call Hindsight Divergence Minimization (HDM). Intuitively, this additional loss imitates an action when the value functions believe that this action can move the agent closer to the goal by at least $-\log \gamma_{\text{hdm}}$ steps. The idea of \textit{imitating the best actions} is similar to self-imitation learning (SIL) \citep{sil, alphastar}, but we study a (reward-free) goal-reaching setting, with the advantage function in SIL being replaced by the \textit{delta} of step counts a particular action can produce in getting closer to the goal.
\begin{figure*}[h]
    \centering
    \includegraphics[width=0.8\textwidth]{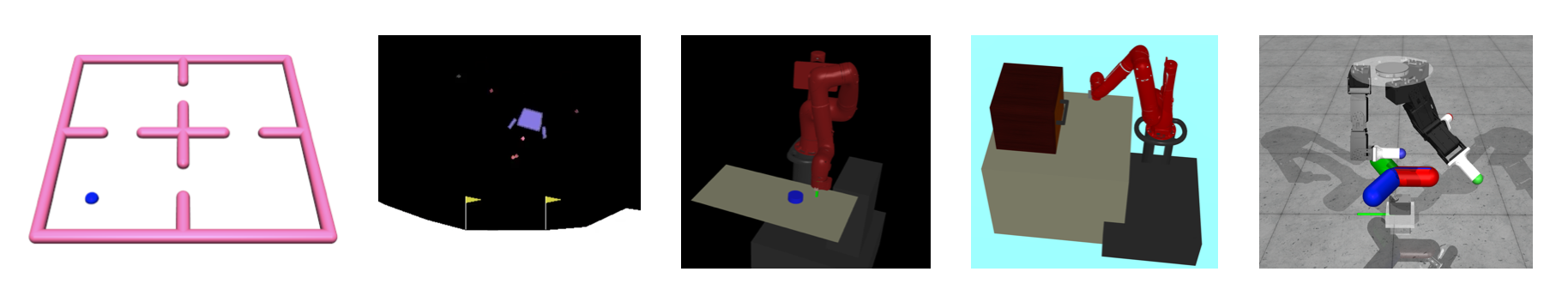}
    \caption{\footnotesize Goal-reaching environments from GCSL \citep{gcsl} that we consider for discrete action space experiments: reaching a goal location in \textit{Four Rooms}, landing at a goal location in \textit{Lunar Lander}, pushing a puck to a goal location in \textit{Sawyer Push}, opening the door to a goal angle in \textit{Door Open} \citep{rig}, turning a valve to a goal orientation in \textit{Claw Manipulate} \citep{robel}. Importantly, those environments have \textbf{discretized action space}. We use {softmax policy} where $Q$-values are policy logits.}
    \label{fig:experiment-environments}
\end{figure*}

\section{Related Work}

\paragraph{Frameworks for Goal Reaching} 
Many attempts have been made to rigorously formulate a theoretical framework for the goal reaching problem. The earliest work that aims to formulate goal-conditioned $Q$-functions as a step count is \citep{learning-to-achieve-goals}, though it only considers discrete state space. The framework introduced in \citep{rewriting-history-with-inverse-rl} considers a finite horizon MDP where the agent receives a $-\infty$ reward at the final timestep if the final state does not match the goal and a $0$ reward in all other conditions, which does not match how multi-goal $Q$-learning works in practice. Hindsight EM \citep{hindsight-em} borrows concepts from control as inference \citep{control-as-inference} but only considers hindsight BC for policy learning. The difficulty of defining the objective of HER \citep{her} discussed in \ref{sec:reach-goals-fewer-steps} was also previously noted in C-learning \citep{c-learning}, which concludes that "\textit{it is unclear what quantity Q-learning with hindsight relabeling optimizes}" and regards the $Q$-values of HER for continuous states as ill-defined. C-learning then proposed to directly estimate \eqref{eq:future-state-distribution} as a workaround. Similarly, \citep{outcome-driven} advocates directly fitting a dynamics model and using per-step reward $\log p(s'=g\mid s,a)$ for goal-reaching. However, those frameworks do not explain the success of simple binary rewards for goal reaching in continuous state space \citep{her, discern}, which still achieves state-of-the-art results today \citep{actionablemodels}. By contrast, our work clearly defines what hindsight relabeling under HER reward optimizes by illustrating its underlying graphical models and the meaning of its $Q$-values, both of which remain well-defined in the case of continuous state space. 
\paragraph{Connecting Goal Reaching with Inverse RL} Classical inverse RL either uses a max margin loss in the apprenticeship learning formulation or contrastive divergence loss \citep{contrastive-divergence} in the maximum entropy formulation \citep{algorithms-for-irl, apprenticeship, ziebart2008maximum, ho-irl}, and those ideas have been borrowed accordingly in the goal reaching literature \citep{contrastive-rl, outcome-driven}. GAIL \citep{gail} uses a GAN-like discriminator \citep{gan} as the reward, and \citep{hbc} augments binary rewards in HER with smoother GAIL rewards for goal reaching. The introduction of $f$-GAN \citep{fgan} generalizes GAN to more $f$-divergences, which led to a similar generalization of GAIL \citep{il-divergence, f-gail}, which led to $f$-divergence based exploration strategy \citep{f-divergence-exploration} and a goal reaching algorithm without hindsight relabeling \citep{how-far-i-ll-go}. SQIL \citep{sqil} is an imitation learning method that runs $Q$-learning on two constant rewards: $r=1$ for expert data and $r=0$ for policy data; SQIL can often outperform GAIL. The SQIL reward is similar to the HER reward in some ways, but prior works have not shown how this type of reward does imitation in goal reaching.

\section{Experiments}

\begin{figure}[h]
    \centering
    \includegraphics[width=0.45\textwidth]{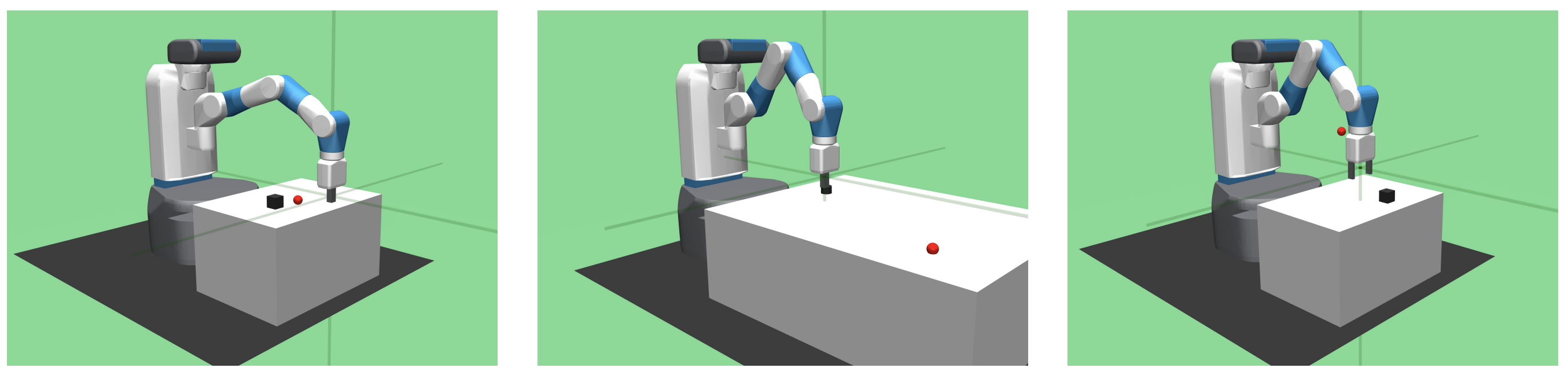}
    \caption{\footnotesize Goal reaching environments from \citep{multi-goal-envs} we consider for benchmarking multi-goal rewards: HER \citep{her} with $(-1, 0)$ rewards, AM \citep{actionablemodels} with $(0,1)$ rewards, HER with $(0,1)$ rewards.}
    \label{fig:fetch-envs}
\end{figure}

\begin{figure*}[h]
    \centering
    \includegraphics[width=0.9\textwidth]{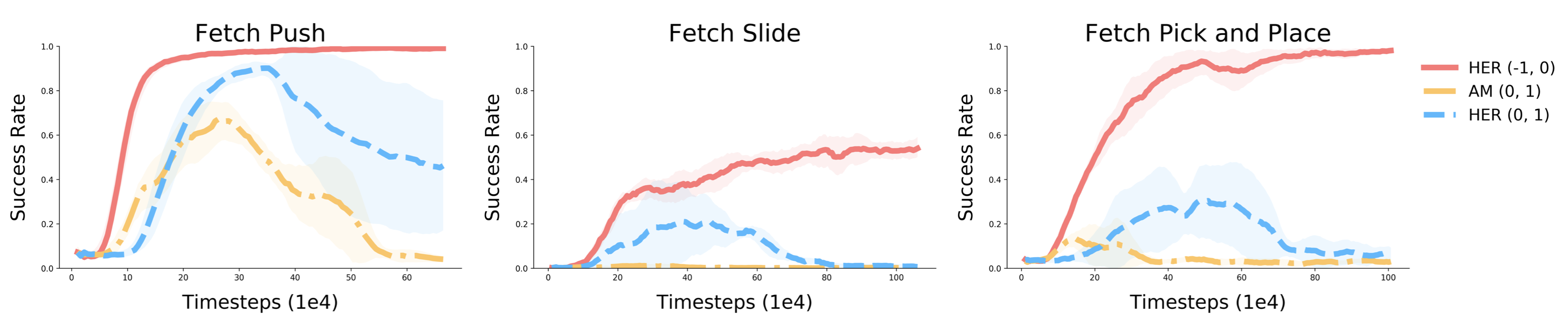}
    \caption{{Success rate comparisons between different multi-goal RL rewards on Fetch environments.}  The only difference between three methods is the reward and backup strategies; all other parts of implementation are the same. Results over 5 seeds are shown. \textbf{Using a reward of $\{-1, 0\}$ is crucial for learning success.}}
    \label{fig:learning-curve-different-rewards}
\end{figure*}

\begin{figure*}[h]
    \centering
    \includegraphics[width=0.9\textwidth]{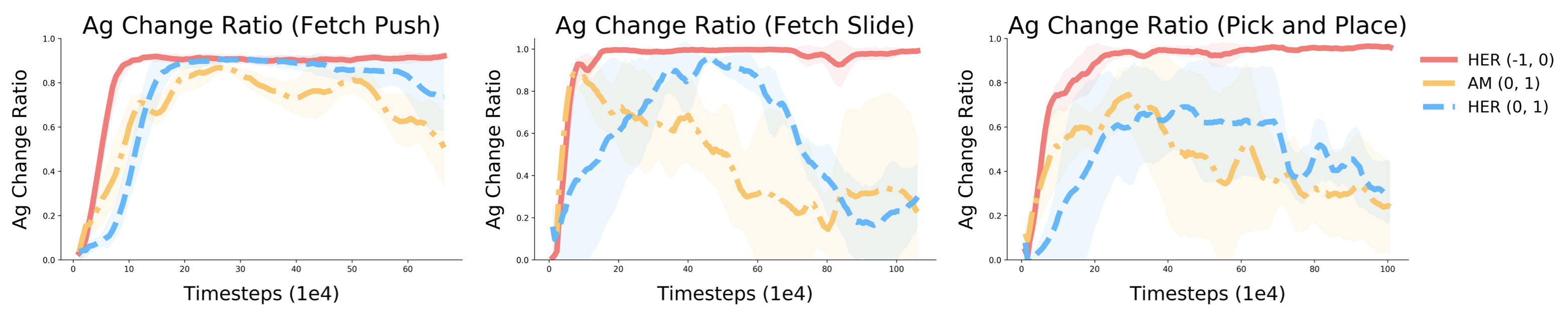}
    \caption{We find that the quantity \textit{ag change ratio} is strongly indicative of learning progress in multi-goal RL. 
    We define the \textit{ag change ratio} of $\pi$ as: the percentage of trajectories where the achieved goals (\textit{ag}) in initial states {$s_{0}$} are different from the achieved goals in final states {$s_{T}$} under $\pi$. Most training signals are only created from \textit{ag} changes, because they provide examples of how to {rearrange} an environment. An increase in ag change ratio often \textit{precedes} an increase in success rates.}
    \label{fig:ag-change-ratio-curve}
\end{figure*}

\begin{figure*}[h]
    \centering
    \includegraphics[width=0.8\textwidth]{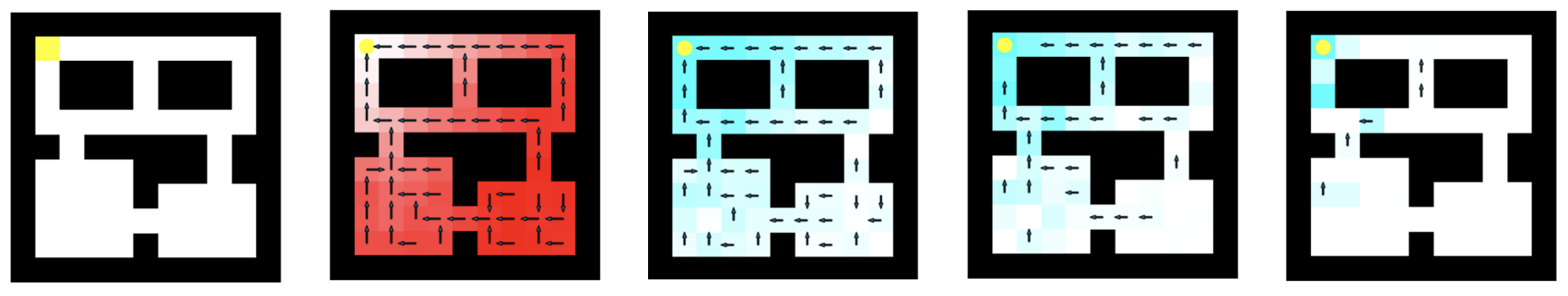}
    \caption{\footnotesize Intuitions about when HDM applies BC. From left to right: (a) the Maze environment with a goal in the upper-left corner; (b) the $Q$-values learned through the converged policy. Lighter red means higher Q-value; (c) visualizing the actions (from the replay) that get imitated when conditioning on the goal and setting $\gamma_{\text{hdm}}=0.95$, with the background color reflecting how much an action moves the agent closer to the goal \textit{based on the agent's own estimate}; (d) $\gamma_{\text{hdm}}=0.85$; (e) $\gamma_{\text{hdm}}=0.75$. As we lower $\gamma_{\text{hdm}}$, the threshold $-\log \gamma_{\text{hdm}}$ gets higher and fewer actions get imitated, with the remaining imitated actions more concentrated around the goal. HDM uses $Q$-learning to account for the worst while imitating the best during the goal-reaching process.}
    \label{fig:maze-visualization}
\end{figure*}

There are two parts of our experimental investigations.
\begin{itemize}
    \item We study how reward designs affect the goal reaching performance, and demonstrate that \textbf{in multi-goal RL, seemingly insignificant details about rewards can lead to the striking difference between success and total failure in learning}. The experimental results are in accordance with our theoretical framework which suggests that $\{-1, 0\}$ rewards should work best.
    \item In the second part, we study the special case of a softmax policy on discrete action space, and find that despite advances in Hindsight BC, HER can still achieve significantly better results. Moreover, a vanilla combination of HER + HBC hurts performance, but our proposed HDM loss (a variant of BC) can further improve the performance of HER on discrete action space. 
\end{itemize}

\subsection{Investigating Reward Designs in Goal Reaching}
\label{sec:reward-design-results}
\begin{table*}[h]

\addtolength{\tabcolsep}{-1pt}
\small
\centering
\begin{tabular}{c@{\hspace{5pt}}lr@{\hspace{-1pt}}lr@{\hspace{-1pt}}lr@{\hspace{-1pt}}lr@{\hspace{-1pt}}lr@{\hspace{-1pt}}l}
\toprule
\multicolumn{1}{l}{}               &    Success Rate (\%)        & \multicolumn{2}{c}{Four Rooms}                                     & \multicolumn{2}{c}{Lunar Landar}              & \multicolumn{2}{c}{Sawyer Push}               & \multicolumn{2}{c}{Door Opening}                                    & \multicolumn{2}{c}{Claw Manipulate} \\
\midrule
\multirow{2}{*} & GCSL / HBC & \colorbox{white}{78.27}   & {\color[HTML]{525252} $\pm$4.76}  & \colorbox{white}{50.00}                              & {\color[HTML]{525252} $\pm$7.77} & \colorbox{white}{44.67}  & {\color[HTML]{525252} $\pm$13.86}  & \colorbox{white}{19.10}  & {\color[HTML]{525252} $\pm$5.97}  & \colorbox{white}{16.80} & {\color[HTML]{525252} $\pm$6.55} \\

&HER $r=(0, 1)$ & \colorbox{white}{86.60}   & {\color[HTML]{525252} $\pm$4.22}  & \colorbox{white}{39.30}                              & {\color[HTML]{525252} $\pm$6.81} & \colorbox{white}{57.60}  & {\color[HTML]{525252} $\pm$6.61}  & \colorbox{white}{82.50}  & {\color[HTML]{525252} $\pm$4.87}  & \colorbox{white}{22.80} & {\color[HTML]{525252} $\pm$6.43} \\

& HER + SQL& \colorbox{white}{88.50}   & {\color[HTML]{525252} $\pm$4.56}  & \colorbox{white}{44.50}                              & {\color[HTML]{525252} $\pm$9.61} & \colorbox{white}{57.20}  & {\color[HTML]{525252} $\pm$6.32}  & \colorbox{white}{84.70}  & {\color[HTML]{525252} $\pm$5.33}  & \colorbox{white}{16.13} & {\color[HTML]{525252} $\pm$8.43} \\

&HER $r=(-1, 0)$ & \colorbox{white}{86.40}   & {\color[HTML]{525252} $\pm$5.11}  & \colorbox{white}{50.80}                              & {\color[HTML]{525252} $\pm$4.66} & \colorbox{white}{54.60}  & {\color[HTML]{525252} $\pm$6.16}  & \colorbox{white}{83.76}  & {\color[HTML]{525252} $\pm$6.02}  & \colorbox{white}{20.20} & {\color[HTML]{525252} $\pm$6.23}\\

&HER + HBC& \colorbox{white}{82.90}   & {\color[HTML]{525252} $\pm$6.24}  & \colorbox{white}{35.33}                              & {\color[HTML]{525252} $\pm$4.57} & \colorbox{white}{52.63}  & {\color[HTML]{525252} $\pm$8.05}  & \colorbox{white}{76.44}  & {\color[HTML]{525252} $\pm$5.37}  & \colorbox{white}{16.93} & {\color[HTML]{525252} $\pm$8.03}\\

&HDM (ours) & \colorbox{mine}{96.27}   & {\color[HTML]{525252} $\pm$2.56}  & \colorbox{mine}{57.60}                              & {\color[HTML]{525252} $\pm$7.21} & \colorbox{mine}{66.00}  & {\color[HTML]{525252} $\pm$5.13}  & \colorbox{mine}{88.60}  & {\color[HTML]{525252} $\pm$4.63}  & \colorbox{mine}{27.89} & {\color[HTML]{525252} $\pm$6.46} \\

\bottomrule
\end{tabular}
\vspace{6pt}
\caption{\footnotesize Benchmark results of test-time success rates in self-supervised goal-reaching, over 5 seeds. We compare our method HDM with GCSL \citep{gcsl}, HER \citep{her} with two different types of rewards, SQL \citep{equivalence-pg-sql} (Soft Q-Learning) + HER, and HER + HBC. We find that HER can still outperforms GCSL / HBC, and a vanilla combination of HER + HBC actually hurts performance. HDM selectively decides on what to imitate and outperforms HER and HBC.} \label{table:results}

\vspace{-3pt}

\end{table*}

We investigate 3 different reward designs for goal reaching:

\textbf{Hindsight Experience Replay (HER)} \citep{her} with $\{-1, 0\}$ rewards:
\begin{align}
    Q(s,a,g) = \begin{cases}
      \gamma \cdot Q(s',\pi(s',g), g) & s'=g \\
      -1 + \gamma \cdot Q(s',\pi(s',g), g) & s' \neq g 
    \end{cases}
\end{align}
\textbf{Actionable Models (AM)} \citep{actionablemodels} uses $\{0, 1\}$ rewards, but directly defines $Q(s,a,s')$ as $1$:
\begin{align}
    Q(s,a,g) = \begin{cases}
      1 & s'=g \\
      \gamma \cdot Q(s',\pi(s',g), g) & s' \neq g 
    \end{cases}
\end{align}
\textbf{HER with $\{0, 1\}$ rewards} which is also a modified version of actionable models where the bellman backup continues after the goal is reached:
\begin{align}
    Q(s,a,g) = \begin{cases}
      1 + \gamma \cdot Q(s',\pi(s',g), g) & s'=g \\
      \gamma \cdot Q(s',\pi(s',g), g) & s' \neq g 
    \end{cases}
\end{align}
We use widely benchmarked Fetch environments \citep{multi-goal-envs} for multi-goal RL as the task suite to compare those three strategies. For all three methods, we use the same set of hyper-parameters (including learning rate, batch size, network architecture, target network update frequency and polyak value, etc). The only differences are the rewards and bellman backups . The results are presented in Figure \ref{fig:learning-curve-different-rewards}. 
\paragraph{Why does $\{-1, 0\}$ reward work so well?} Our framework shows that running $Q$-learning with $\{-1, 0\}$ rewards leads to a probabilistically well-defined $Q$ values in the form of a normalized step count \eqref{eq:goal-reaching-step-count}. One can perhaps interpret HER with $\{0, 1\}$ rewards as adding a constant offset $1 / (1 - \gamma)$ on the converged $Q$-values of $\{-1, 0\}$ rewards, but this offset is quite large and eventually causes learning to diverge in this case. Actionable Models (AM) address this issue by stopping further bellman backup once the goal is reached $(s'=g)$ and directly setting the $Q$-value to be $1$. This does not align with our analysis of the divergence minimization objective \eqref{eq:hdm-divergence-minimization}, and fails to teach the policy to stop at the goal state once the goal is reached: under this bellman backup, once the goal is reached once, nothing matters afterwards. To summarize, using a reward of $\{-1,0\}$ is crucial for the success of multi-goal $Q$-learning with hindsight relabeling. 


\subsection{Goal Reaching with Discrete Action Space}

\label{subsection:self-supervised-setup}

In this section, we study whether our proposed HDM loss \eqref{eq:hdm-final-loss} can facilitate the goal-conditioned $Q$-learning process on environments with either discrete or discretized action space (Figure \ref{fig:experiment-environments}), as mentioned in \ref{sec:hindsight-bc-helps}. Hindsight BC (or GCSL) with discrete action space is known to have promising results even in a totally self-supervised goal reaching setting without external demonstrations \citep{gcsl}. To provide a fair comparison to HBC, we do not assume that the agent has direct access to the ground truth binary reward metric. Note the original HER uses this metric during relabeling \citep{her}, but assuming access to this feedback is often unrealistic for real-world robot-learning \citep{rl-without-gt-state, actionablemodels}. Instead, a positive reward is provided only when the relabeled hindsight goal is the immediate next state. As a result, the training procedure is completely self-supervised, similar to HBC. Additionally, we also consider the HER + Soft Q-Learning (SQL) \citep{equivalence-pg-sql} baseline, since SQL often improves policy robustness \citep{sac}. HDM builds on top of HER with $(-1,0)$ rewards and adds a BC-like loss with a clipping condition \eqref{eq:hdm-final-loss} such that only the actions that move an agent closer to the goal get imitated. See Figure \ref{fig:maze-visualization} for more visualizations of when HDM applies the BC loss on top of a $Q$-learning process. 

The results in Table \ref{table:results} show that HDM achieves the strongest performance on all environments, while reducing variances in success rates. Interestingly, there is no consensus best baseline algorithm, with all five algorithms achieving good results in some environments and subpar performance in others. Combining HER with HBC (which blindly imitates all actions in hindsight) produces worse results than not imitating at all and only resorting to value learning. In contrast, HDM allows for better control over what to imitate. Further ablations on HDM are provided in the \ref{fig:ablation-studies} of the appendix, which shows that HDM outperforms HER and GCSL across a variety of $\gamma_{\text{hdm}}$ values. Our results indicate that BC alone can only serve as an auxiliary (or perhaps pre-training) objective; value learning is needed to implicitly model the future \eqref{eq:goal-reaching-fewer-steps} and improve the policy.

\section{Conclusion}

This work presents a novel goal reaching framework that exploits the deep connection between multi-goal RL and inverse RL to derive a family of goal-conditioned RL algorithms. By understanding hindsight goal relabeling from a divergence minimization perspective, our framework reveals the importance of reward design in multi-goal RL, which is found to significantly affect learning performance in our experiments. Furthermore, we propose an additional hindsight divergence minimization (HDM) loss, which uses $Q$-learning to account for the worst while using BC to imitate the best, and demonstrate its superior performance on discrete action space. In the future, we hope to further develop our framework to explicitly account for exploration.

\bibliography{example_paper}

\begin{thebibliography}{74}
\providecommand{\natexlab}[1]{#1}
\providecommand{\url}[1]{\texttt{#1}}
\expandafter\ifx\csname urlstyle\endcsname\relax
  \providecommand{\doi}[1]{doi: #1}\else
  \providecommand{\doi}{doi: \begingroup \urlstyle{rm}\Url}\fi

\bibitem[Abbeel \& Ng(2004)Abbeel and Ng]{apprenticeship}
Abbeel, P. and Ng, A.~Y.
\newblock Apprenticeship learning via inverse reinforcement learning.
\newblock In \emph{Proceedings of the twenty-first international conference on
  Machine learning}, pp.\ ~1, 2004.

\bibitem[Ahn et~al.(2020)Ahn, Zhu, Hartikainen, Ponte, Gupta, Levine, and
  Kumar]{robel}
Ahn, M., Zhu, H., Hartikainen, K., Ponte, H., Gupta, A., Levine, S., and Kumar,
  V.
\newblock Robel: Robotics benchmarks for learning with low-cost robots.
\newblock In \emph{Conference on robot learning}, pp.\  1300--1313. PMLR, 2020.

\bibitem[Ahn et~al.(2022)Ahn, Brohan, Brown, Chebotar, Cortes, David, Finn,
  Gopalakrishnan, Hausman, Herzog, et~al.]{say-can}
Ahn, M., Brohan, A., Brown, N., Chebotar, Y., Cortes, O., David, B., Finn, C.,
  Gopalakrishnan, K., Hausman, K., Herzog, A., et~al.
\newblock Do as i can, not as i say: Grounding language in robotic affordances.
\newblock \emph{arXiv preprint arXiv:2204.01691}, 2022.

\bibitem[Andrychowicz et~al.(2017)Andrychowicz, Wolski, Ray, Schneider, Fong,
  Welinder, McGrew, Tobin, Abbeel, and Zaremba]{her}
Andrychowicz, M., Wolski, F., Ray, A., Schneider, J., Fong, R., Welinder, P.,
  McGrew, B., Tobin, J., Abbeel, P., and Zaremba, W.
\newblock Hindsight experience replay.
\newblock \emph{Advances in neural information processing systems}, 30, 2017.

\bibitem[Batra et~al.(2020)Batra, Chang, Chernova, Davison, Deng, Koltun,
  Levine, Malik, Mordatch, Mottaghi, et~al.]{rearrangement}
Batra, D., Chang, A.~X., Chernova, S., Davison, A.~J., Deng, J., Koltun, V.,
  Levine, S., Malik, J., Mordatch, I., Mottaghi, R., et~al.
\newblock Rearrangement: A challenge for embodied ai.
\newblock \emph{arXiv preprint arXiv:2011.01975}, 2020.

\bibitem[Blier et~al.(2021)Blier, Tallec, and
  Ollivier]{learning-successor-states}
Blier, L., Tallec, C., and Ollivier, Y.
\newblock Learning successor states and goal-dependent values: A mathematical
  viewpoint.
\newblock \emph{arXiv preprint arXiv:2101.07123}, 2021.

\bibitem[Brohan et~al.(2022)Brohan, Brown, Carbajal, Chebotar, Dabis, Finn,
  Gopalakrishnan, Hausman, Herzog, Hsu, et~al.]{robotics-transformer}
Brohan, A., Brown, N., Carbajal, J., Chebotar, Y., Dabis, J., Finn, C.,
  Gopalakrishnan, K., Hausman, K., Herzog, A., Hsu, J., et~al.
\newblock Rt-1: Robotics transformer for real-world control at scale.
\newblock \emph{arXiv preprint arXiv:2212.06817}, 2022.

\bibitem[Brown et~al.(2020)Brown, Mann, Ryder, Subbiah, Kaplan, Dhariwal,
  Neelakantan, Shyam, Sastry, Askell, et~al.]{gpt3}
Brown, T., Mann, B., Ryder, N., Subbiah, M., Kaplan, J.~D., Dhariwal, P.,
  Neelakantan, A., Shyam, P., Sastry, G., Askell, A., et~al.
\newblock Language models are few-shot learners.
\newblock \emph{Advances in neural information processing systems},
  33:\penalty0 1877--1901, 2020.

\bibitem[Chebotar et~al.(2021)Chebotar, Hausman, Lu, Xiao, Kalashnikov, Varley,
  Irpan, Eysenbach, Julian, Finn, and Levine]{actionablemodels}
Chebotar, Y., Hausman, K., Lu, Y., Xiao, T., Kalashnikov, D., Varley, J.,
  Irpan, A., Eysenbach, B., Julian, R., Finn, C., and Levine, S.
\newblock Actionable models: Unsupervised offline reinforcement learning of
  robotic skills.
\newblock \emph{arXiv preprint arXiv:2104.07749}, 2021.

\bibitem[Ding et~al.(2019)Ding, Florensa, Abbeel, and Phielipp]{hbc}
Ding, Y., Florensa, C., Abbeel, P., and Phielipp, M.
\newblock Goal-conditioned imitation learning.
\newblock In \emph{Advances in Neural Information Processing Systems}, pp.\
  15298--15309, 2019.

\bibitem[Duan et~al.(2017)Duan, Andrychowicz, Stadie, Jonathan~Ho, Schneider,
  Sutskever, Abbeel, and Zaremba]{one-shot-imitation}
Duan, Y., Andrychowicz, M., Stadie, B., Jonathan~Ho, O., Schneider, J.,
  Sutskever, I., Abbeel, P., and Zaremba, W.
\newblock One-shot imitation learning.
\newblock \emph{Advances in neural information processing systems}, 30, 2017.

\bibitem[Durugkar et~al.(2021)Durugkar, Tec, Niekum, and
  Stone]{f-divergence-exploration}
Durugkar, I., Tec, M., Niekum, S., and Stone, P.
\newblock Adversarial intrinsic motivation for reinforcement learning.
\newblock \emph{Advances in Neural Information Processing Systems},
  34:\penalty0 8622--8636, 2021.

\bibitem[Eysenbach et~al.(2020{\natexlab{a}})Eysenbach, Geng, Levine, and
  Salakhutdinov]{rewriting-history-with-inverse-rl}
Eysenbach, B., Geng, X., Levine, S., and Salakhutdinov, R.~R.
\newblock Rewriting history with inverse rl: Hindsight inference for policy
  improvement.
\newblock \emph{Advances in neural information processing systems},
  33:\penalty0 14783--14795, 2020{\natexlab{a}}.

\bibitem[Eysenbach et~al.(2020{\natexlab{b}})Eysenbach, Salakhutdinov, and
  Levine]{c-learning}
Eysenbach, B., Salakhutdinov, R., and Levine, S.
\newblock C-learning: Learning to achieve goals via recursive classification.
\newblock \emph{arXiv preprint arXiv:2011.08909}, 2020{\natexlab{b}}.

\bibitem[Eysenbach et~al.(2022)Eysenbach, Zhang, Salakhutdinov, and
  Levine]{contrastive-rl}
Eysenbach, B., Zhang, T., Salakhutdinov, R., and Levine, S.
\newblock Contrastive learning as goal-conditioned reinforcement learning.
\newblock \emph{arXiv preprint arXiv:2206.07568}, 2022.

\bibitem[Finn et~al.(2016)Finn, Christiano, Abbeel, and
  Levine]{finn2016connection}
Finn, C., Christiano, P., Abbeel, P., and Levine, S.
\newblock A connection between generative adversarial networks, inverse
  reinforcement learning, and energy-based models.
\newblock \emph{arXiv preprint arXiv:1611.03852}, 2016.

\bibitem[Florensa et~al.(2017)Florensa, Held, Wulfmeier, Zhang, and
  Abbeel]{reverse-curriculum}
Florensa, C., Held, D., Wulfmeier, M., Zhang, M., and Abbeel, P.
\newblock Reverse curriculum generation for reinforcement learning.
\newblock In \emph{Conference on robot learning}, pp.\  482--495. PMLR, 2017.

\bibitem[Fu et~al.(2017)Fu, Luo, and Levine]{airl}
Fu, J., Luo, K., and Levine, S.
\newblock Learning robust rewards with adversarial inverse reinforcement
  learning.
\newblock \emph{arXiv preprint arXiv:1710.11248}, 2017.

\bibitem[Fujimoto et~al.(2018)Fujimoto, Hoof, and Meger]{td3}
Fujimoto, S., Hoof, H., and Meger, D.
\newblock Addressing function approximation error in actor-critic methods.
\newblock In \emph{International conference on machine learning}, pp.\
  1587--1596. PMLR, 2018.

\bibitem[Garg et~al.(2021)Garg, Chakraborty, Cundy, Song, and Ermon]{iq-learn}
Garg, D., Chakraborty, S., Cundy, C., Song, J., and Ermon, S.
\newblock Iq-learn: Inverse soft-q learning for imitation.
\newblock \emph{Advances in Neural Information Processing Systems}, 34, 2021.

\bibitem[Ghasemipour et~al.(2020)Ghasemipour, Zemel, and Gu]{il-divergence}
Ghasemipour, S. K.~S., Zemel, R., and Gu, S.
\newblock A divergence minimization perspective on imitation learning methods.
\newblock In \emph{Conference on Robot Learning}, pp.\  1259--1277. PMLR, 2020.

\bibitem[Ghosh et~al.(2019)Ghosh, Gupta, Fu, Reddy, Devin, Eysenbach, and
  Levine]{gcsl}
Ghosh, D., Gupta, A., Fu, J., Reddy, A., Devin, C., Eysenbach, B., and Levine,
  S.
\newblock Learning to reach goals without reinforcement learning.
\newblock \emph{ArXiv}, abs/1912.06088, 2019.

\bibitem[Goodfellow et~al.(2014)Goodfellow, Pouget-Abadie, Mirza, Xu,
  Warde-Farley, Ozair, Courville, and Bengio]{gan}
Goodfellow, I., Pouget-Abadie, J., Mirza, M., Xu, B., Warde-Farley, D., Ozair,
  S., Courville, A., and Bengio, Y.
\newblock Generative adversarial nets.
\newblock \emph{Advances in neural information processing systems}, 27, 2014.

\bibitem[Haarnoja et~al.(2018)Haarnoja, Zhou, Abbeel, and Levine]{sac}
Haarnoja, T., Zhou, A., Abbeel, P., and Levine, S.
\newblock Soft actor-critic: Off-policy maximum entropy deep reinforcement
  learning with a stochastic actor.
\newblock In \emph{International conference on machine learning}, pp.\
  1861--1870. PMLR, 2018.

\bibitem[Hinton(2002)]{contrastive-divergence}
Hinton, G.~E.
\newblock Training products of experts by minimizing contrastive divergence.
\newblock \emph{Neural computation}, 14\penalty0 (8):\penalty0 1771--1800,
  2002.

\bibitem[Ho \& Ermon(2016)Ho and Ermon]{gail}
Ho, J. and Ermon, S.
\newblock Generative adversarial imitation learning.
\newblock \emph{Advances in neural information processing systems}, 29, 2016.

\bibitem[Ho et~al.(2016)Ho, Gupta, and Ermon]{ho-irl}
Ho, J., Gupta, J., and Ermon, S.
\newblock Model-free imitation learning with policy optimization.
\newblock In \emph{International Conference on Machine Learning}, pp.\
  2760--2769. PMLR, 2016.

\bibitem[Hong et~al.(2022)Hong, Yang, and Agrawal]{bvn}
Hong, Z.-W., Yang, G., and Agrawal, P.
\newblock Bilinear value networks.
\newblock \emph{arXiv preprint arXiv:2204.13695}, 2022.

\bibitem[Jang et~al.(2021)Jang, Irpan, Khansari, Kappler, Ebert, Lynch, Levine,
  and Finn]{bc-zero}
Jang, E., Irpan, A., Khansari, M., Kappler, D., Ebert, F., Lynch, C., Levine,
  S., and Finn, C.
\newblock Bc-z: Zero-shot task generalization with robotic imitation learning.
\newblock \emph{ArXiv}, abs/2202.02005, 2021.

\bibitem[Janner et~al.(2020)Janner, Mordatch, and Levine]{gamma-models}
Janner, M., Mordatch, I., and Levine, S.
\newblock gamma-models: Generative temporal difference learning for
  infinite-horizon prediction.
\newblock \emph{Advances in Neural Information Processing Systems},
  33:\penalty0 1724--1735, 2020.

\bibitem[Kaelbling(1993)]{learning-to-achieve-goals}
Kaelbling, L.~P.
\newblock Learning to achieve goals.
\newblock In \emph{IJCAI}, volume~2, pp.\  1094--8. Citeseer, 1993.

\bibitem[Ke et~al.(2021)Ke, Choudhury, Barnes, Sun, Lee, and
  Srinivasa]{ke2021imitation}
Ke, L., Choudhury, S., Barnes, M., Sun, W., Lee, G., and Srinivasa, S.
\newblock Imitation learning as f-divergence minimization.
\newblock In \emph{Algorithmic Foundations of Robotics XIV: Proceedings of the
  Fourteenth Workshop on the Algorithmic Foundations of Robotics 14}, pp.\
  313--329. Springer, 2021.

\bibitem[Kingma \& Ba(2014)Kingma and Ba]{adam}
Kingma, D.~P. and Ba, J.
\newblock Adam: A method for stochastic optimization.
\newblock \emph{arXiv preprint arXiv:1412.6980}, 2014.

\bibitem[Kostrikov et~al.(2019)Kostrikov, Nachum, and Tompson]{valuedice}
Kostrikov, I., Nachum, O., and Tompson, J.
\newblock Imitation learning via off-policy distribution matching.
\newblock \emph{arXiv preprint arXiv:1912.05032}, 2019.

\bibitem[Levine(2018)]{control-as-inference}
Levine, S.
\newblock Reinforcement learning and control as probabilistic inference:
  Tutorial and review.
\newblock \emph{arXiv preprint arXiv:1805.00909}, 2018.

\bibitem[Lillicrap et~al.(2015)Lillicrap, Hunt, Pritzel, Heess, Erez, Tassa,
  Silver, and Wierstra]{ddpg}
Lillicrap, T.~P., Hunt, J.~J., Pritzel, A., Heess, N., Erez, T., Tassa, Y.,
  Silver, D., and Wierstra, D.
\newblock Continuous control with deep reinforcement learning.
\newblock \emph{arXiv preprint arXiv:1509.02971}, 2015.

\bibitem[Lin et~al.(2019)Lin, Baweja, and Held]{rl-without-gt-state}
Lin, X., Baweja, H.~S., and Held, D.
\newblock Reinforcement learning without ground-truth state.
\newblock \emph{arXiv preprint arXiv:1905.07866}, 2019.

\bibitem[Lynch \& Sermanet(2020)Lynch and Sermanet]{lynch2020language}
Lynch, C. and Sermanet, P.
\newblock Language conditioned imitation learning over unstructured data.
\newblock \emph{arXiv preprint arXiv:2005.07648}, 2020.

\bibitem[Lynch et~al.(2019)Lynch, Khansari, Xiao, Kumar, Tompson, Levine, and
  Sermanet]{latent-plan-from-play}
Lynch, C., Khansari, M., Xiao, T., Kumar, V., Tompson, J., Levine, S., and
  Sermanet, P.
\newblock Learning latent plans from play.
\newblock In \emph{CoRL}, 2019.

\bibitem[Ma et~al.(2022)Ma, Yan, Jayaraman, and Bastani]{how-far-i-ll-go}
Ma, Y.~J., Yan, J., Jayaraman, D., and Bastani, O.
\newblock How far i'll go: Offline goal-conditioned reinforcement learning via
  $ f $-advantage regression.
\newblock \emph{arXiv preprint arXiv:2206.03023}, 2022.

\bibitem[Mnih et~al.(2015)Mnih, Kavukcuoglu, Silver, Rusu, Veness, Bellemare,
  Graves, Riedmiller, Fidjeland, Ostrovski, et~al.]{dqn}
Mnih, V., Kavukcuoglu, K., Silver, D., Rusu, A.~A., Veness, J., Bellemare,
  M.~G., Graves, A., Riedmiller, M., Fidjeland, A.~K., Ostrovski, G., et~al.
\newblock Human-level control through deep reinforcement learning.
\newblock \emph{nature}, 518\penalty0 (7540):\penalty0 529--533, 2015.

\bibitem[Nachum \& Dai(2020)Nachum and Dai]{rl-fenchel-rockafellar}
Nachum, O. and Dai, B.
\newblock Reinforcement learning via fenchel-rockafellar duality.
\newblock \emph{arXiv preprint arXiv:2001.01866}, 2020.

\bibitem[Nachum et~al.(2019{\natexlab{a}})Nachum, Chow, Dai, and Li]{dualdice}
Nachum, O., Chow, Y., Dai, B., and Li, L.
\newblock Dualdice: Behavior-agnostic estimation of discounted stationary
  distribution corrections.
\newblock \emph{Advances in Neural Information Processing Systems}, 32,
  2019{\natexlab{a}}.

\bibitem[Nachum et~al.(2019{\natexlab{b}})Nachum, Dai, Kostrikov, Chow, Li, and
  Schuurmans]{algaedice}
Nachum, O., Dai, B., Kostrikov, I., Chow, Y., Li, L., and Schuurmans, D.
\newblock Algaedice: Policy gradient from arbitrary experience.
\newblock \emph{arXiv preprint arXiv:1912.02074}, 2019{\natexlab{b}}.

\bibitem[Nair et~al.(2018{\natexlab{a}})Nair, McGrew, Andrychowicz, Zaremba,
  and Abbeel]{overcome-explore}
Nair, A., McGrew, B., Andrychowicz, M., Zaremba, W., and Abbeel, P.
\newblock Overcoming exploration in reinforcement learning with demonstrations.
\newblock In \emph{2018 IEEE international conference on robotics and
  automation (ICRA)}, pp.\  6292--6299. IEEE, 2018{\natexlab{a}}.

\bibitem[Nair et~al.(2018{\natexlab{b}})Nair, Pong, Dalal, Bahl, Lin, and
  Levine]{rig}
Nair, A.~V., Pong, V., Dalal, M., Bahl, S., Lin, S., and Levine, S.
\newblock Visual reinforcement learning with imagined goals.
\newblock \emph{Advances in neural information processing systems}, 31,
  2018{\natexlab{b}}.

\bibitem[Ng et~al.(2000)Ng, Russell, et~al.]{algorithms-for-irl}
Ng, A.~Y., Russell, S., et~al.
\newblock Algorithms for inverse reinforcement learning.
\newblock In \emph{Icml}, volume~1, pp.\ ~2, 2000.

\bibitem[Nguyen et~al.(2010)Nguyen, Wainwright, and Jordan]{convex-risk}
Nguyen, X., Wainwright, M.~J., and Jordan, M.~I.
\newblock Estimating divergence functionals and the likelihood ratio by convex
  risk minimization.
\newblock \emph{IEEE Transactions on Information Theory}, 56\penalty0
  (11):\penalty0 5847--5861, 2010.

\bibitem[Nowozin et~al.(2016)Nowozin, Cseke, and Tomioka]{fgan}
Nowozin, S., Cseke, B., and Tomioka, R.
\newblock f-gan: Training generative neural samplers using variational
  divergence minimization.
\newblock \emph{Advances in neural information processing systems}, 29, 2016.

\bibitem[Oh et~al.(2018)Oh, Guo, Singh, and Lee]{sil}
Oh, J., Guo, Y., Singh, S., and Lee, H.
\newblock Self-imitation learning.
\newblock In \emph{International Conference on Machine Learning}, pp.\
  3878--3887. PMLR, 2018.

\bibitem[OpenAI et~al.(2021)OpenAI, Plappert, Sampedro, Xu, Akkaya, Kosaraju,
  Welinder, D'Sa, Petron, Pinto, et~al.]{openai2021asymmetric}
OpenAI, O., Plappert, M., Sampedro, R., Xu, T., Akkaya, I., Kosaraju, V.,
  Welinder, P., D'Sa, R., Petron, A., Pinto, H. P. d.~O., et~al.
\newblock Asymmetric self-play for automatic goal discovery in robotic
  manipulation.
\newblock \emph{arXiv preprint arXiv:2101.04882}, 2021.

\bibitem[Pitis et~al.(2020)Pitis, Chan, Zhao, Stadie, and Ba]{mega}
Pitis, S., Chan, H., Zhao, S., Stadie, B., and Ba, J.
\newblock Maximum entropy gain exploration for long horizon multi-goal
  reinforcement learning.
\newblock In \emph{International Conference on Machine Learning}, pp.\
  7750--7761. PMLR, 2020.

\bibitem[Plappert et~al.(2018{\natexlab{a}})Plappert, Andrychowicz, Ray,
  McGrew, Baker, Powell, Schneider, Tobin, Chociej, Welinder,
  et~al.]{multi-goal-envs}
Plappert, M., Andrychowicz, M., Ray, A., McGrew, B., Baker, B., Powell, G.,
  Schneider, J., Tobin, J., Chociej, M., Welinder, P., et~al.
\newblock Multi-goal reinforcement learning: Challenging robotics environments
  and request for research.
\newblock \emph{arXiv preprint arXiv:1802.09464}, 2018{\natexlab{a}}.

\bibitem[Plappert et~al.(2018{\natexlab{b}})Plappert, Andrychowicz, Ray,
  McGrew, Baker, Powell, Schneider, Tobin, Chociej, Welinder,
  et~al.]{multi-goal-rl}
Plappert, M., Andrychowicz, M., Ray, A., McGrew, B., Baker, B., Powell, G.,
  Schneider, J., Tobin, J., Chociej, M., Welinder, P., et~al.
\newblock Multi-goal reinforcement learning: Challenging robotics environments
  and request for research.
\newblock \emph{arXiv preprint arXiv:1802.09464}, 2018{\natexlab{b}}.

\bibitem[Pomerleau(1988)]{alvinn}
Pomerleau, D.~A.
\newblock Alvinn: An autonomous land vehicle in a neural network.
\newblock \emph{Advances in neural information processing systems}, 1, 1988.

\bibitem[Pong et~al.(2018)Pong, Gu, Dalal, and Levine]{tdm}
Pong, V., Gu, S., Dalal, M., and Levine, S.
\newblock Temporal difference models: Model-free deep rl for model-based
  control.
\newblock \emph{arXiv preprint arXiv:1802.09081}, 2018.

\bibitem[Pong et~al.(2019)Pong, Dalal, Lin, Nair, Bahl, and Levine]{skew-fit}
Pong, V.~H., Dalal, M., Lin, S., Nair, A., Bahl, S., and Levine, S.
\newblock Skew-fit: State-covering self-supervised reinforcement learning.
\newblock \emph{arXiv preprint arXiv:1903.03698}, 2019.

\bibitem[Reddy et~al.(2019)Reddy, Dragan, and Levine]{sqil}
Reddy, S., Dragan, A.~D., and Levine, S.
\newblock Sqil: Imitation learning via reinforcement learning with sparse
  rewards.
\newblock \emph{arXiv preprint arXiv:1905.11108}, 2019.

\bibitem[Ross et~al.(2011)Ross, Gordon, and Bagnell]{dagger}
Ross, S., Gordon, G.~J., and Bagnell, J.~A.
\newblock A reduction of imitation learning and structured prediction to
  no-regret online learning.
\newblock In \emph{AISTATS}, 2011.

\bibitem[Rudner et~al.(2021)Rudner, Pong, McAllister, Gal, and
  Levine]{outcome-driven}
Rudner, T.~G., Pong, V., McAllister, R., Gal, Y., and Levine, S.
\newblock Outcome-driven reinforcement learning via variational inference.
\newblock \emph{Advances in Neural Information Processing Systems},
  34:\penalty0 13045--13058, 2021.

\bibitem[Schulman et~al.(2017{\natexlab{a}})Schulman, Chen, and
  Abbeel]{equivalence-pg-sql}
Schulman, J., Chen, X., and Abbeel, P.
\newblock Equivalence between policy gradients and soft q-learning.
\newblock \emph{arXiv preprint arXiv:1704.06440}, 2017{\natexlab{a}}.

\bibitem[Schulman et~al.(2017{\natexlab{b}})Schulman, Wolski, Dhariwal,
  Radford, and Klimov]{ppo}
Schulman, J., Wolski, F., Dhariwal, P., Radford, A., and Klimov, O.
\newblock Proximal policy optimization algorithms.
\newblock \emph{arXiv preprint arXiv:1707.06347}, 2017{\natexlab{b}}.

\bibitem[Shafiullah et~al.(2022)Shafiullah, Cui, Altanzaya, and
  Pinto]{bahvior-transformer}
Shafiullah, N. M.~M., Cui, Z.~J., Altanzaya, A., and Pinto, L.
\newblock Behavior transformers: Cloning $ k $ modes with one stone.
\newblock \emph{arXiv preprint arXiv:2206.11251}, 2022.

\bibitem[Silver et~al.(2014)Silver, Lever, Heess, Degris, Wierstra, and
  Riedmiller]{deterministic-policy-gradient}
Silver, D., Lever, G., Heess, N., Degris, T., Wierstra, D., and Riedmiller, M.
\newblock Deterministic policy gradient algorithms.
\newblock In \emph{International conference on machine learning}, pp.\
  387--395. PMLR, 2014.

\bibitem[Silver et~al.(2021)Silver, Singh, Precup, and
  Sutton]{reward-is-enough}
Silver, D., Singh, S., Precup, D., and Sutton, R.~S.
\newblock Reward is enough.
\newblock \emph{Artificial Intelligence}, 299:\penalty0 103535, 2021.

\bibitem[Sutton et~al.(1999)Sutton, McAllester, Singh, and
  Mansour]{policy-gradient-methods}
Sutton, R.~S., McAllester, D., Singh, S., and Mansour, Y.
\newblock Policy gradient methods for reinforcement learning with function
  approximation.
\newblock \emph{Advances in neural information processing systems}, 12, 1999.

\bibitem[Tang \& Kucukelbir(2021)Tang and Kucukelbir]{hindsight-em}
Tang, Y. and Kucukelbir, A.
\newblock Hindsight expectation maximization for goal-conditioned reinforcement
  learning.
\newblock In \emph{International Conference on Artificial Intelligence and
  Statistics}, pp.\  2863--2871. PMLR, 2021.

\bibitem[Van~Hasselt et~al.(2016)Van~Hasselt, Guez, and Silver]{double-dqn}
Van~Hasselt, H., Guez, A., and Silver, D.
\newblock Deep reinforcement learning with double q-learning.
\newblock In \emph{Proceedings of the AAAI conference on artificial
  intelligence}, volume~30, 2016.

\bibitem[Vinyals et~al.(2019)Vinyals, Babuschkin, Czarnecki, Mathieu, Dudzik,
  Chung, Choi, Powell, Ewalds, Georgiev, et~al.]{alphastar}
Vinyals, O., Babuschkin, I., Czarnecki, W.~M., Mathieu, M., Dudzik, A., Chung,
  J., Choi, D.~H., Powell, R., Ewalds, T., Georgiev, P., et~al.
\newblock Grandmaster level in starcraft ii using multi-agent reinforcement
  learning.
\newblock \emph{Nature}, 575\penalty0 (7782):\penalty0 350--354, 2019.

\bibitem[Warde-Farley et~al.(2018)Warde-Farley, Van~de Wiele, Kulkarni,
  Ionescu, Hansen, and Mnih]{discern}
Warde-Farley, D., Van~de Wiele, T., Kulkarni, T., Ionescu, C., Hansen, S., and
  Mnih, V.
\newblock Unsupervised control through non-parametric discriminative rewards.
\newblock \emph{arXiv preprint arXiv:1811.11359}, 2018.

\bibitem[Zhang et~al.(2021)Zhang, Yang, and Stadie]{world-model-graph}
Zhang, L., Yang, G., and Stadie, B.~C.
\newblock World model as a graph: Learning latent landmarks for planning.
\newblock In \emph{International Conference on Machine Learning}, pp.\
  12611--12620. PMLR, 2021.

\bibitem[Zhang et~al.(2020)Zhang, Li, Zhang, and Zhang]{f-gail}
Zhang, X., Li, Y., Zhang, Z., and Zhang, Z.-L.
\newblock f-gail: Learning f-divergence for generative adversarial imitation
  learning.
\newblock \emph{Advances in neural information processing systems},
  33:\penalty0 12805--12815, 2020.

\bibitem[Zhu et~al.(2020)Zhu, Lin, Dai, and Zhou]{OPOLO}
Zhu, Z., Lin, K., Dai, B., and Zhou, J.
\newblock Off-policy imitation learning from observations.
\newblock \emph{Advances in Neural Information Processing Systems},
  33:\penalty0 12402--12413, 2020.

\bibitem[Ziebart et~al.(2008)Ziebart, Maas, Bagnell, Dey,
  et~al.]{ziebart2008maximum}
Ziebart, B.~D., Maas, A.~L., Bagnell, J.~A., Dey, A.~K., et~al.
\newblock Maximum entropy inverse reinforcement learning.
\newblock In \emph{Aaai}, volume~8, pp.\  1433--1438. Chicago, IL, USA, 2008.

\end{thebibliography}
\bibliographystyle{icml2023}

\newpage
\appendix
\onecolumn

\section{Proofs}

\subsection{Deriving Hindsight Behavior Cloning}
\label{sec:hbc-proof}
We want to prove the equivalence defined in \eqref{hbc-derivation}
\begin{equation*}
\begin{aligned}
&\argmin_{\pi} \mathcal{D}_{KL} (p_{\mu}(s, a, s^{+}) \parallel p_{\pi}^{\text{BC}}(s, a, g) ) \\
= &\argmin_{\pi} \E_{ \rho_{\mu}(s, a) p^{+}_{\mu} (s^{+} \mid s, a)}  [ -\log \pi(a \mid s, g) ]
\end{aligned}
\end{equation*}
We use the definition of KL divergence:
\begin{align*}
    D_{KL}(P \parallel Q) = \int_{X} P(x) \log \dfrac{P(x)}{Q(x)} \mathrm{d}x
\end{align*}
And the graphical models of the two distributions:
\begin{align*}
    p_{\mu}(s, a, s^{+}) &= \rho_{\mu}(s, a) p^{+}_{\mu} (s^{+} \mid s, a) \\
    p_{\pi}^{\text{BC}}(s, a, g) &= {p(g)}  \rho_{\mu}(s) \pi(a\mid s, g) 
\end{align*}
Resulting in
\begin{align*}
    D_{KL}(p_{\mu}(s, a, s^{+}) \parallel p_{\pi}^{\text{BC}}(s, a, g)) &= \E_{p_{\mu}(s, a, s^{+})}[ \log p_{\mu}(s, a, s^{+}) - \log p_{\pi}^{\text{BC}}(s, a, g) ] \\
    &= \E_{p_{\mu}(s, a, s^{+})}[ \log p_{\mu}(s, a, s^{+}) - \log {p(g)} - \log \rho_{\mu}(s) - \log \pi(a\mid s, g)  ] \\
    &= \E_{\rho_{\mu}(s, a) p^{+}_{\mu} (s^{+} \mid s, a)} \bigg[ \log \dfrac{p_{\mu}(s, a, s^{+})}{p(g)\rho_{\mu}(s)} - \log \pi(a\mid s, g)  \bigg]
\end{align*}
where we find that 
\begin{equation*}
\begin{aligned}
\argmin_{\pi} \mathcal{D}_{KL} (p_{\mu}(s, a, s^{+}) \parallel p_{\pi}^{\text{BC}}(s, a, g) ) = \argmin_{\pi} \E_{ \rho_{\mu}(s, a) p^{+}_{\mu} (s^{+} \mid s, a)}  [ -\log \pi(a \mid s, g) ]
\end{aligned}
\end{equation*}

\subsection{Main Lemmas}

\begin{lemma}[Online-to-offline transformation for goal reaching]
Given a goal-conditioned policy\\ $\pi(a\mid s,g)$, its corresponding $Q$-function $Q^{\pi}(s,a,g)$, and arbitrary state-action visitation distribution $\rho_{\mu}(s,a)$ of another policy $\mu(a\mid s)$, the expected temporal difference for online rollouts under $\pi$  is:
\[ \E_{ p(g) {\color{blue}\rho_{\pi}(s, a\mid g)} } [(Q^{\pi}-\gamma \cdot \mathcal{P}^{\pi}Q^{\pi})({\color{blue} s}, {\color{blue} a}, g)] = \E_{ p(g) {\color{red}{\rho_{\mu}(s,a)}} {\color{blue} \pi(\tilde{a}\mid s,g)}  } [Q^{\pi}({\color{red} s}, {\color{blue} \tilde{a}}, g) - \gamma \cdot \mathcal{P}^{\pi}Q^{\pi}({\color{red} s}, {\color{red} a}, g) ]
\]
\end{lemma}

\begin{proof}[Proof of Lemma \ref{lemma:online-to-offline}]

\begin{align*}
    & \E_{ p(g) {\color{blue}\rho_{\pi}(s, a\mid g)} } [(Q^{\pi}-\gamma \cdot \mathcal{P}^{\pi}Q^{\pi})({\color{blue} s}, {\color{blue} a}, g)] \\
    &= \E_{ p(g) \color{blue}{\rho_{\pi}(s, a\mid g)} } [Q^{\pi}(s, a, g) - \gamma \E_{p(s'\mid s,a), \pi(a'\mid s', g)} Q^{\pi}(s', a', g) ] \\
    &= (1-\gamma) \sum_{t=0}^{\infty} \gamma^{t} \E_{ \substack{ p(g) \color{blue}{\rho^{t}_{\pi} (s \mid g)} \\ \pi(a\mid s, g)} } \Big[ Q^{\pi}(s,a,g) - \gamma \E_{\substack{p(s'\mid s,a)\\ \pi(a'\mid s', g)} } Q^{\pi}(s',a',g) \Big] \\
    &= (1-\gamma) \sum_{t}^{\infty} \Bigg\{ \gamma^{t} \E_{ \substack{ p(g) \\ \color{blue}{\rho^{t}_{\pi} (s\mid g)} \\ \pi(a\mid s,g)} }[Q^{\pi}(s,a,g)] - \gamma^{t+1} \E_{ \substack{ p(g) \\ \color{blue}{\rho^{t+1}_{\pi} (s \mid g)} \\ \pi(a\mid s,g)}}[Q^{\pi}(s,a,g)] \Bigg\} \\
    &= (1-\gamma) \E_{p(g), {\color{red}{\rho^{0}(s)}}, \pi(a\mid s,g)} [Q^{\pi}(s,a,g)] \\
    &= (1-\gamma) \sum_{t}^{\infty} \Bigg\{ \gamma^{t} \E_{ \substack{ p(g) \\ \color{red}{\rho^{t}_{\mu} (s)} \\ \pi(a\mid s,g)} }[Q^{\pi}(s,a,g)] - \gamma^{t+1} \E_{ \substack{ p(g) \\ \color{red}{\rho^{t+1}_{\mu} (s)} \\ \pi(a\mid s,g)}}[Q^{\pi}(s,a,g)] \Bigg\} \\
    &= (1-\gamma) \sum_{t=0}^{\infty} \gamma^{t} \E_{ \substack{ p(g) \color{red}{\rho^{t}_{\mu} (s,a)} \\ {\color{blue} \pi(\tilde{a}\mid s,g)} } }[ Q^{\pi}(s,{\color{blue}\tilde{a}},g) - \gamma \E_{\substack{{\color{red} p(s'\mid s,a)} \\ \pi(a'\mid s',g)} } Q^{\pi}(s',a',g) ] \\
    &= \E_{ p(g) {\color{red}{\rho_{\mu}(s,a)}} {\color{blue} \pi(\tilde{a}\mid s,g)}  } [Q^{\pi}(s, {\color{blue} \tilde{a}}, g) - \gamma \E_{{\color{red} p(s'\mid s,a)}, \pi(a'\mid s', g)} Q^{\pi}(s', a', g) ]
\end{align*}

\end{proof}

\subsection{$Q$ Function as Step Counts}

We make the following definitions:
\begin{align*}
    Q_{t} &= Q(s_{t}, a_{t}, g) \\
    Q_{t+1} &= \E_{\pi(s_{t+1},g)}[Q(s_{t+1}, a', g)] \\
    \beta_{t} &= p(s_{t+1}=g\mid s_t, a_t) \\
    \alpha_{t}^{(i)} &= \sum_{\Delta=i} \gamma^{\Delta} p_{\pi}(s_{t+1+\Delta} = g \mid s_{t}, a_{t})
\end{align*}
And make the following mild assumptions:
\begin{align*}
    1/\E_{\pi(a\mid s_{t},g)}[\alpha_{t}^{(0)}] &= \E_{\pi(a\mid s_{t}, g)}[1 / \alpha_{t}^{(0)}] \\
    1/\E_{p(s_{t+1}\mid s_{t}, a_{t})}[\alpha_{t+1}^{(1)}] &= \E_{p(s_{t+1}\mid s_{t}, a_{t})}[1 / \alpha_{t}^{(1)}]
\end{align*}
Which can be combined to reach the result:
\begin{align*}
    1 / \alpha_{t}^{(1)} = \E_{p(s_{t+1}\mid s_{t}, a_{t}) \pi(a_{t+1}'\mid s_{t+1},g) } [1 / \alpha_{t+1}^{(1)}]
\end{align*}
Meaning that for an infinite horizon MDP where all goals are eventually reached, the reciprocal of a geometrically summed future likelihood of reaching a goal remains approximately the same in expectation under rollouts.

\begin{lemma}[Recursive estimate of goal reaching step count]
Under the notations and the assumptions above, the step count definition \eqref{eq:goal-reaching-step-count} has the following property:
\[ 
    Q_{t} = \E_{p(s_{t+1}\mid s_t, a_t)}[(1-\dfrac{\beta_{t} }{ \beta_{t} + \alpha_{t}^{(1)} } ) ( -1 + Q_{t+1} )] 
\]
\end{lemma}

\begin{proof}[Proof of Lemma \ref{lemma:goal-reaching-step-count}]

We start from the definition of our $Q$ value:
\begin{align*}
    Q(s_{t}, a_{t}, g) &= - \dfrac{ \sum_{\Delta=0} \gamma^{\Delta} p_{\pi}(s_{t+1+\Delta} = g \mid s_{t}, a_{t}) \cdot \Delta }{ \sum_{\Delta=0} \gamma^{\Delta} p_{\pi}(s_{t+1+\Delta} = g \mid s_{t}, a_{t}) } 
\end{align*}
$Q^{\pi}(s_{t}, a_{t}, g)$ can be expanded as
\begin{align*}
    - \dfrac{ \gamma p(s_{t+2}=g \mid s_t, a_t) + 2\gamma^{2}p(s_{t+3}=g \mid s_t, a_t) + 3\gamma^{3}p(s_{t+4}=g \mid s_t, a_t) + \cdots } { p(s_{t+1}=g \mid s_t, a_t) + \gamma p(s_{t+2}=g \mid s_t, a_t) + \gamma^{2}p(s_{t+3}=g \mid s_t, a_t) + \gamma^{3}p(s_{t+4}=g \mid s_t, a_t) + \cdots  } 
\end{align*}
The value function $V(s_t, g)$ is defined to be:
\begin{align*}
    V(s_{t}, g) &= - \dfrac{ \sum_{\Delta=0} \gamma^{\Delta} p_{\pi}(s_{t+1+\Delta} = g \mid s_{t}) \cdot \Delta }{ \sum_{\Delta=0} \gamma^{\Delta} p_{\pi}(s_{t+1+\Delta} = g \mid s_{t}) } 
\end{align*}
Which can be expanded as:
\begin{align*}
    - \dfrac{ \gamma p(s_{t+2}=g \mid s_t) + 2\gamma^{2}p(s_{t+3}=g \mid s_t) + 3\gamma^{3}p(s_{t+4}=g \mid s_t) + \cdots } { p(s_{t+1}=g \mid s_t) + \gamma p(s_{t+2}=g \mid s_t) + \gamma^{2}p(s_{t+3}=g \mid s_t) + \gamma^{3}p(s_{t+4}=g \mid s_t) + \cdots  }
\end{align*}
On the other hand, $\E_{\pi(a\mid s_{t}, g)}[Q(s_{t}, a, g)]$ can be written as:
\begin{align*}
    - \E_{\pi(a\mid s_{t}, g)} \bigg[ \dfrac{ \gamma p(s_{t+2}=g \mid s_t, a) + 2\gamma^{2}p(s_{t+3}=g \mid s_t, a) + 3\gamma^{3}p(s_{t+4}=g \mid s_t, a) + \cdots } { p(s_{t+1}=g \mid s_t, a) + \gamma p(s_{t+2}=g \mid s_t, a) + \gamma^{2}p(s_{t+3}=g \mid s_t, a) + \gamma^{3}p(s_{t+4}=g \mid s_t, a) + \cdots } \bigg]
\end{align*}
by the law of total probability, 
\begin{align*}
    p(s_{t+1+\Delta} = g \mid s_{t}) = \E_{\pi(a\mid s_{t},g)}[p(s_{t+1+\Delta} = g \mid s_{t}, a_{t})]
\end{align*} 
Which means that $V(s_t, g) = \E_{\pi(a\mid s_{t},g)}[Q(s_{t}, a, g)]$ under the given assumption that $1/\E_{\pi(a\mid s_{t},g)}[\alpha_{t}^{(0)}] = \E_{\pi(a\mid s_{t})}[1 / \alpha_{t}^{(0)}]$:
\begin{align}
    \E_{\pi(a\mid s_{t}, g)} \bigg[ \dfrac{1}{ \sum_{\Delta=0} \gamma^{\Delta} p(s_{t+1+\Delta} = g \mid s_{t}, a ) } \bigg] = \dfrac{1}{ \sum_{\Delta=0} \gamma^{\Delta} p(s_{t+1+\Delta} = g \mid s_{t} )} \label{eq:assumption-wrt-policy}
\end{align}

Furthermore, $V^{\pi}(s_{t+1}, g)$ can be expanded as:
\begin{align*}
    - \dfrac{ \gamma p(s_{t+3} =g \mid s_{t+1}) + 2\gamma^{2}p(s_{t+4}=g \mid s_{t+1}) + 3\gamma^{3}p(s_{t+5}=g \mid s_{t+1}) + \cdots }{ p(s_{t+2}=g \mid s_{t+1}) + \gamma p(s_{t+3}=g \mid s_{t+1}) + \gamma^{2}p(s_{t+4}=g \mid s_{t+1}) + \gamma^{3}p(s_{t+5}=g \mid s_{t+1}) + \cdots  }
\end{align*}
Note that $-1 + V^{\pi}(s_{t+1}, g)$ becomes
\begin{align*}
    - \dfrac{ p(s_{t+2}=g \mid s_{t+1}) + 2\gamma p(s_{t+3}=g \mid s_{t+1}) + 3\gamma^{2}p(s_{t+4}=g \mid s_{t+1}) + 4\gamma^{3}p(s_{t+5}=g \mid s_{t+1}) \cdots } { p(s_{t+2}=g \mid s_{t+1}) + \gamma p(s_{t+3}=g \mid s_{t+1}) + \gamma^{2}p(s_{t+4}=g \mid s_{t+1}) + \gamma^{3}p(s_{t+5}=g \mid s_{t+1}) + \cdots  } 
\end{align*}
Multiply both the nominator and the denominator by $\gamma$:
\begin{align*}
    - \dfrac{ \gamma p(s_{t+2}=g \mid s_{t+1}) + 2\gamma^{2} p(s_{t+3}=g \mid s_{t+1}) + 3\gamma^{3}p(s_{t+4}=g \mid s_{t+1}) + \cdots }{ \gamma p(s_{t+2}=g \mid s_{t+1}) + \gamma^{2} p(s_{t+3}=g \mid s_{t+1}) + \gamma^{3}p(s_{t+4}=g \mid s_{t+1}) + \cdots }
\end{align*}
Note that, for $\Delta > 0$, by the law of total probability,
\begin{align*}
    p(s_{t+1+\Delta} = g \mid s_{t}, a_{t}) = \E_{p(s_{t+1} \mid s_{t}, a_{t})}[p(s_{t+1+\Delta} = g \mid s_{t+1})]
\end{align*}
As a result:
\begin{align*}
    \E_{p(s_{t+1} \mid s_{t}, a_{t})}[-1 + V^{\pi}(s_{t+1}, g)] &= - Z [\gamma p(s_{t+2}=g \mid s_t, a_t) + 2\gamma^{2} p(s_{t+3}=g \mid s_t, a_t) + 3\gamma^{3}p(s_{t+4}=g \mid s_t, a_t) + \cdots ]
\end{align*}
where 
\begin{align*}
    Z &= \E_{p(s_{t+1} \mid s_{t}, a_{t})} \bigg[ \dfrac{1}{ \sum_{\Delta=1} \gamma^{\Delta} p(s_{t+1+\Delta} = g \mid s_{t+1} ) } \bigg]
\end{align*}
Resulting in:
\begin{align*}
    \dfrac{ Q^{\pi}(s_t, a_t, g) }{\E_{p(s_{t+1} \mid s_{t}, a_{t})}[-1 + V^{\pi}(s_{t+1}, g)]} = \dfrac{ \dfrac{1}{ \E_{p(s_{t+1} \mid s_t, a_t)} \bigg[\dfrac{1}{\gamma p(s_{t+2}=g \mid s_{t+1}) + \gamma^{2} p(s_{t+3}=g \mid s_{t+1})+ \cdots} \bigg] } } { p(s_{t+1}=g \mid s_t, a_t) + \gamma p(s_{t+2}=g \mid s_t, a_t) + \gamma^{2}p(s_{t+3}=g \mid s_t, a_t)+ \cdots  } 
\end{align*}
Under the given assumption that $1/\E_{p(s_{t+1}\mid s_{t}, a_{t})}[\alpha_{t+1}^{(1)}] = \E_{p(s_{t+1}\mid s_{t}, a_{t})}[1 / \alpha_{t}^{(1)}]$: 
\begin{align}
    \E_{p(s_{t+1} \mid s_{t}, a_{t})} \bigg[ \dfrac{1}{ \sum_{\Delta=1} \gamma^{\Delta} p(s_{t+1+\Delta} = g \mid s_{t+1} ) } \bigg] = \dfrac{1}{ \sum_{\Delta=1} \gamma^{\Delta} p(s_{t+1+\Delta} = g \mid s_{t}, a_{t} )} \label{eq:assumption-wrt-dynamics}
\end{align}
we have:
\begin{align*}
    \dfrac{ Q^{\pi}(s_t, a_t, g) }{\E_{p(s_{t+1} \mid s_{t}, a_{t})}[-1 + V^{\pi}(s_{t+1}, g)]} &= \dfrac{ \gamma p(s_{t+2}=g \mid s_{t}, a_{t}) + \gamma^{2} p(s_{t+3}=g \mid s_{t}, a_{t})+ \cdots } { p(s_{t+1}=g \mid s_t, a_t) + \gamma p(s_{t+2}=g \mid s_t, a_t) + \gamma^{2}p(s_{t+3}=g \mid s_t, a_t)+ \cdots } \\
    &= \bigg( 1 - \dfrac{ p(s_{t+1}=g \mid s_t, a_t) } { p(s_{t+1}=g \mid s_t, a_t) + \sum_{\Delta=1} \gamma^{\Delta} p(s_{t+1+\Delta} = g \mid s_{t}, a_{t}) }  \bigg) \\
    &= \bigg(1 - \dfrac{\beta_{t}}{\beta_{t} + \alpha_{t}^{(1)}} \bigg) 
\end{align*}
Which eventually leads to:
\begin{align*}
    Q^{\pi}(s_t, a_t, g) &= \E_{p(s_{t+1} \mid s_{t}, a_{t})} \bigg[\Big(1 - \dfrac{\beta_{t}}{\beta_{t} + \alpha_{t}^{(1)}} \Big) \Big(-1 + \E_{\pi(s_{t+1},g)}[Q^{\pi}(s_{t+1}, \cdot, g)] \Big)\bigg]
\end{align*}
\end{proof}


\subsection{Deriving HER Rewards}
\label{her-reward}

\begin{lemma}[Understanding Hindsight Experience Replay]
Multi-goal $Q$-learning with HER reward \eqref{eq:her-reward} with $\{-1, 0\}$ is a special case of minimizing the following objective:
\[
    \E_{ \substack{\rho_{\mu}(s, a) \\ p(s'\mid s,a) \\ p^{+}_{\mu} (s^{+} \mid s, a) }} [f^{*}(-(Q- \gamma \mathcal{P}^{\pi}Q)(s,a,s^{+})) - \beta Q(s,a,s')]
\]
with $\beta=(1-\gamma)$ and the convex function $f^{*}$ chosen to be:
\[
f^{*}(x) = (x-1)^{2} / 2 + 3/2
\]
\end{lemma}

\begin{proof}[Proof of Lemma \ref{lemma:her-as-a-special-case}]
Recall that we have defined $p^{+}_{\mu} (s^{+} \mid s, a)$ in \eqref{eq:future-state-distribution}:
\begin{align*}
    p^{+}_{\mu} (s^{+} \mid s, a) &= (1-\gamma) p(s^{+} \mid s, a) + \gamma \int_{S\times A} p(s' \mid s, a) \mu(a' \mid s') p^{+}_{\mu} (s^{+} \mid s', a') \mathrm{d} s' \mathrm{d} a'
\end{align*}
And that we have defined a quadratic form of $f^{*}$ (with $\overline{c}$ being constants):
\begin{align*}
    f^{*}(x) &= (x-1)^{2} / 2 + \overline{c}
\end{align*}

Using the dynamics to expand the expectation and applying the choice of $f^{*}$ being a quadratic, the loss becomes:
\begin{align*}
     \argmin_{Q} & \E_{ \substack{\rho_{\mu}(s, a) p(s' \mid s, a)} } (1-\gamma) \cdot \Big[  \dfrac{1}{2}\Big(-1 + (\gamma \mathcal{P}^{\pi}Q - Q_{\theta})(s,a,s')\Big)^{2} - \cdot Q_{\theta}(s,a,s') \Big] \label{eq:her-derive-immediate-next-state} \\
     & + \E_{ \substack{\rho_{\mu}(s, a) p(s' \mid s, a) \mu(a'\mid s') p_{\mu}^{+}(s^{+}\mid s', a') } } \Big[ \gamma \cdot  \dfrac{1}{2}\Big(-1 + (\gamma \mathcal{P}^{\pi}Q - Q_{\theta})(s,a,s^{+})\Big)^{2} \Big]
\end{align*}
\textit{Assuming that there is a stop gradient sign on $\mathcal{P}^{\pi}Q$ because of the use of a target network} \citep{dqn, ddpg, sac}, we can rewrite the above as one single quadratic and see that the gradient of the above loss w.r.t $Q$ is equivalent to the gradient of the following squared Bellman residual:
\begin{align*}
    \argmin_{Q} \E_{ \substack{\rho_{\mu}(s, a) p(s' \mid s,a) p_{\mu}^{+}(s^{+} \mid s, a)} } \Big[ \dfrac{1}{2}\Big(r(s,a,s',s^{+}) + (\gamma \mathcal{P}^{\pi}Q - Q_{\theta})(s,a,s^{+})\Big)^{2} \Big]
\end{align*}
where the reward function $r(s,a,s',s^{+})$ is:
\begin{align*}
    \begin{cases}
      0 & s'=s^{+} \\
     -1 & s'\neq s^{+}
    \end{cases} 
\end{align*}
The constant $3/2$ in $f^{*}$ is chosen to ensure that $((f^{*})^{*})(1) = 0$ in the definition of $f$-divergence \eqref{eq:f-divergence}, but it does not affect the optimization process. 


\end{proof}

\section{Experimental Details}

For the reward design experiments, we use the following hyper-parameters in Table \ref{tab:reward-hparams}, which are mostly the same from prior open-sourced implementations \citep{her, mega, world-model-graph}. For discrete action space experiments, we use the following thresholds (of Euclidean norms) for determining success: $[0.08, 0.08, 0.05, 0.05, 0.1]$, which are tight thresholds based on our visualizations of the environments (those are tighter thresholds than the original ones in \citep{gcsl}; based on our observations, the original thresholds are often too loose). We use the same network architecture, sampling and optimization schedules for all the methods, as described in Table \ref{tab:discrete-actions-hyper-parameters}. As for $\gamma_{\text{HDM}}$, we set it to be $0.85$ in \textit{Four Rooms} and \textit{Lunar Lander}, $0.5$ in \textit{Sawyer Push} and \textit{Claw Manipulate}, and $0.4$ for \textit{Door Opening}. Ablation on this hyper-parameter can be found in Figure \ref{fig:ablation-studies}.

\begin{table}[H]
\renewcommand{\arraystretch}{1.1}
\centering
\caption{Hyper-parameters for the goal reaching reward design experiments}
\label{tab:reward-hparams}
\vspace{1mm}
\begin{tabular}{l l| l }
\toprule
\multicolumn{2}{l|}{Parameter} &  Value\\
\midrule
\multicolumn{2}{l|}{\it{DDPG} \citep{ddpg}}& \\
& optimizer & Adam \citep{adam} \\
& architecture & MLP + BVN \citep{bvn} \\
& number of hidden layers (all networks) & 2 \\
& number of hidden units per layer & 256\\
& nonlinearity & ReLU\\
& Normalize per-dimension obs \citep{ppo} & yes \\
& polyak for target network ($\tau$)& 0.995\\
& target network update interval & 10\\
& Use target network for policy \citep{td3} & yes \\
& ratio between environment vs optimization steps & 2\\
& Random action probability & 0.2 \\ 
& Initial random trajectories per worker & 100 \\
& Hindsight relabelling ratio & 0.85 \\
& Learning rate & 0.001 \\
& Batch size & 1024 \\
& Gamma factor $\gamma$ & 0.99 \\
& Action $L2$ regularization & 0.01 \\
& Gaussian noise scale & 0.1 \\
& Number of parallel workers & 12 \\
& Replay buffer size & 2500000 \\ 
\bottomrule
\end{tabular}
\end{table}

\begin{table}[H]
\renewcommand{\arraystretch}{1.1}
\centering
\caption{Hyper-parameters for discrete action space experiments}
\label{tab:discrete-actions-hyper-parameters}
\vspace{1mm}
\begin{tabular}{l l| l }
\toprule
\multicolumn{2}{l|}{Parameters} &  Value\\
\midrule
\multicolumn{2}{l|}{\it{DDQN} \citep{double-dqn}}& \\
& Optimizer & Adam \citep{adam} \\
& Number of hidden layers (all networks) & 2 \\
& Number of hidden units per layer & [400, 300] \citep{td3} \\
& Non-linearity & ReLU\\
& Polyak for target network & 0.995\\
& Target update interval & 10\\
& Ratio between env vs optimization steps & 1\\
& Initial random trajectories & 200 \\
& Hindsight relabelling ratio & 0.85 \\
& Update every \# of steps in environment & 50 \\
& Next state relabelling ratio & 0.2 \\
& Learning rate & 5.e-4 \\
& BC loss weight & 1.0 \\
& Gamma factor $\gamma$ & 0.98 \\
& Logit temperature for SQL \citep{equivalence-pg-sql} & 0.2 \\
& Batch size & 256 \\
& Epsilon greedy \citep{dqn} & 0.2 \\
& Replay buffer size & 2500000 \\
\bottomrule
\end{tabular}
\end{table}

\begin{figure}[h]
    \centering
    \includegraphics[width=0.7\textwidth]{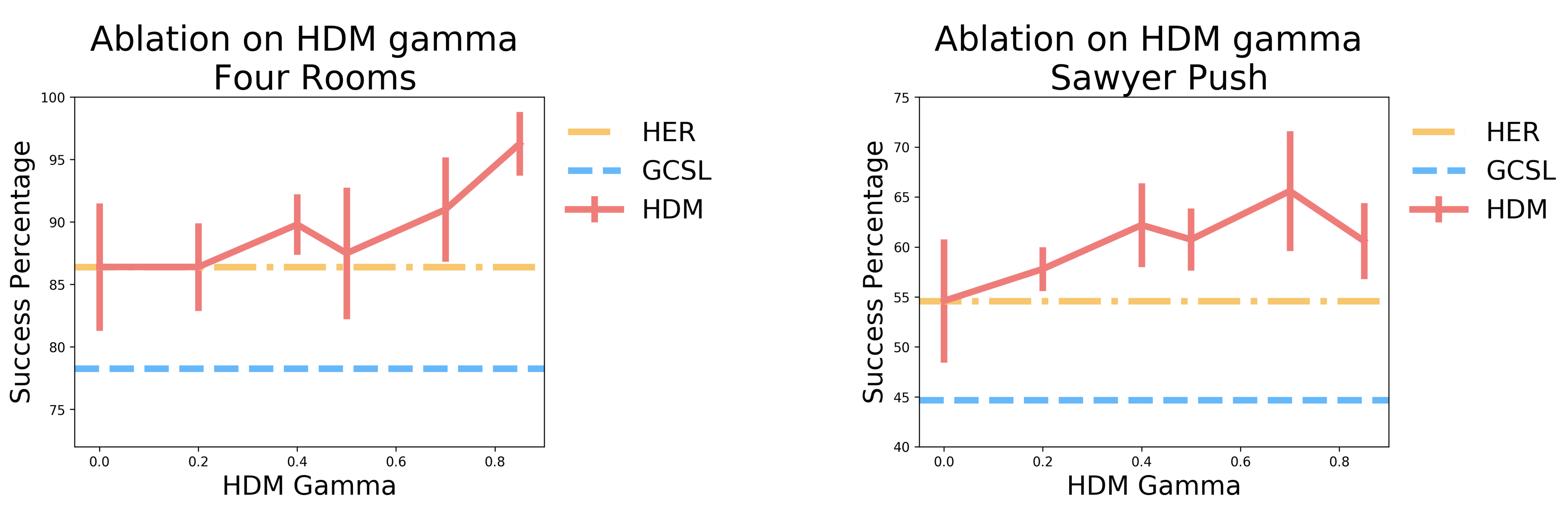}
    \caption{Ablation studies on HDM Gamma $\gamma_{\text{hdm}}$ (see equation \eqref{eq:hdm-final-loss}). The orange line and the blue line denote HER \citep{her} and GCSL \citep{gcsl} baseline performance. Intuitively, $\gamma_{\text{hdm}}$ controls the threshold for when an action is considered good enough for imitation. As we lower $\gamma_{\text{hdm}}$, the threshold $-\log \gamma_{\text{hdm}}$ gets higher and fewer actions get imitated, with the remaining imitated actions more concentrated around the goal, where hindsight-relabeled actions are more likely to be optimal. The ablation shows that HDM outperforms HER and GCSL across a variety of $\gamma_{\text{hdm}}$, while recovering the performance of HER when its value is close to zero. Both environments have discrete action space where we use a softmax (Boltzmann) policy \citep{equivalence-pg-sql} with the softmax logits being $Q$-values.}
    \label{fig:ablation-studies}
\end{figure}

\section{Further Discussions on Achieved Goal (\emph{ag}) Change Ratio}

We define the \textit{ag change ratio} of $\pi$ to be: the percentage of trajectories where the achieved goals in initial states {$s_{0}$} are different from the achieved goals in final states {$s_{T}$} under $\pi$. Using notation from \eqref{eq:her-reward}, it can be computed as $\E_{s_{0} \cdots s_{T} \sim \pi}[- r_{\text{HER}}(\cdot, \cdot, s_{0}, s_{T})]$.

We then define \textit{initial ag change ratio} to be the ag change ratio of a random-acting policy $\pi_{0}$. Using notation from \eqref{eq:her-reward}, it can be computed as $\E_{s_{0} \cdots s_{T} \sim \pi_{0}}[- r_{\text{HER}}(\cdot, \cdot, s_{0}, s_{T})]$.

In our experiments, we have made the following observations:
\begin{itemize}
    \item \textit{ag change ratio} is strongly indicative of learning progress on many environments, as shown in Figure \ref{fig:ag-change-ratio-curve}.
    \item \textit{initial ag change ratio}, which can be computed without training any policies, seems to be correlated to the final performance of HBC / GCSL, as shown in Figure \ref{fig:analysis-linear-fit-hbc} (bubble sizes are based on the variances of the success rate).
\end{itemize}

\begin{wrapfigure}{r}{0.45\textwidth}
\captionsetup{justification=centering}
\begin{minipage}{0.45\textwidth}
\includegraphics[width=0.95\textwidth]{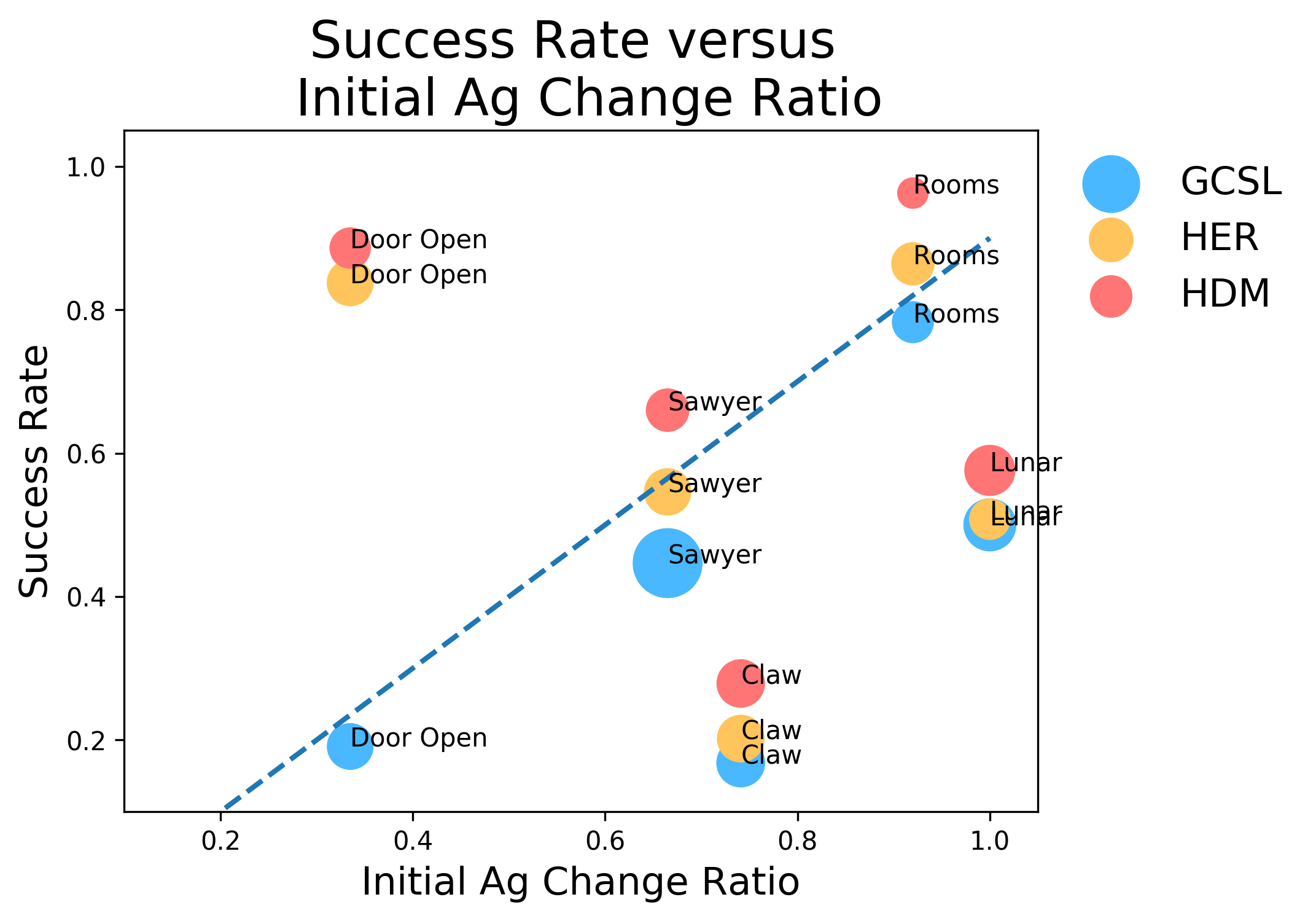}
\caption{\footnotesize Success rate versus \textit{initial ag} change ratio.}
 \label{fig:analysis-linear-fit-hbc}
 \end{minipage}
 \vspace{-1mm}
\end{wrapfigure}

We suspect that the reasons are the following:
\begin{itemize}
    \item Most training signals are only created from \textit{ag} changes, because they provide examples of how to \textit{rearrange} an environment. As the policy learns to rearrange the environment with higher frequency, hindsight relabeling ensures that the learning progress naturally accelerates. 
    \item One likely reason why the performance of GCSL seems to be upper-bounded by a linear relationship between the final success rate and the initial ag change ratio is that: a BC-style objective starts off cloning the initially random trajectories, so if \textit{initial ag change ratio} is low, the policy would not learn to rearrange \textit{ag} from self-imitation, compounding to a low final performance. HER and HDM are able to surpass this upper ceiling likely because they try to reach goals with fewer steps besides imitation.
    \item This finding suggests that in order to make goal-reaching \textit{easier}, we should either modify the initial state distribution $\rho^{0}(s)$ such that \textit{ag} can be easily changed through random exploration \citep{reverse-curriculum} (if the policy is training from scratch), or initialize BC from some high-quality demonstrations where $ag$ does change \citep{hbc, overcome-explore, latent-plan-from-play}.
\end{itemize}

The above insights on the possible relationship between \textit{ag change ratio} and exploration difficulty makes us wonder whether the striking results about reward design effectiveness in Section \ref{sec:reward-design-results} are independent of the exploration problem. Indeed, Figure \ref{fig:ag-change-ratio-curve} shows that HER with $\{-1, 0\}$ rewards takes off faster partially because it manages to increase its \textit{ag change ratio} faster. To show that \textbf{our conclusions about goal-reaching reward design are independent of exploration difficulties}, we benchmark the reward-design results on an environment where all methods have an \textit{initial ag change ratio} of $1$ (and where the \textit{ag change ratio} stays at around $1$ throughout learning for all methods), namely the \textit{Shallow Hand} environment \citep{multi-goal-envs} \texttt{HandManipulateBlockRotateZ}. We use the same set of hyper-parameters as Table \ref{tab:reward-hparams} except for the fact that we use $20$ parallel workers for the \textit{Hand} environments rather than $12$, since this environment is harder. The results are presented in Figure \ref{fig:hand-env-results}, showing that HER with $\{-1, 0\}$ rewards is the only method that learns and succeeds. 

These findings are largely consistent with the conjecture hypothesized in the original Hindsight Experience Replay (HER) paper, which discussed a similar problem in the context of a simple bit-flipping experiment (See \cite{her} Section 3.1).

\begin{figure}
    \centering
    \includegraphics[width=0.95\textwidth]{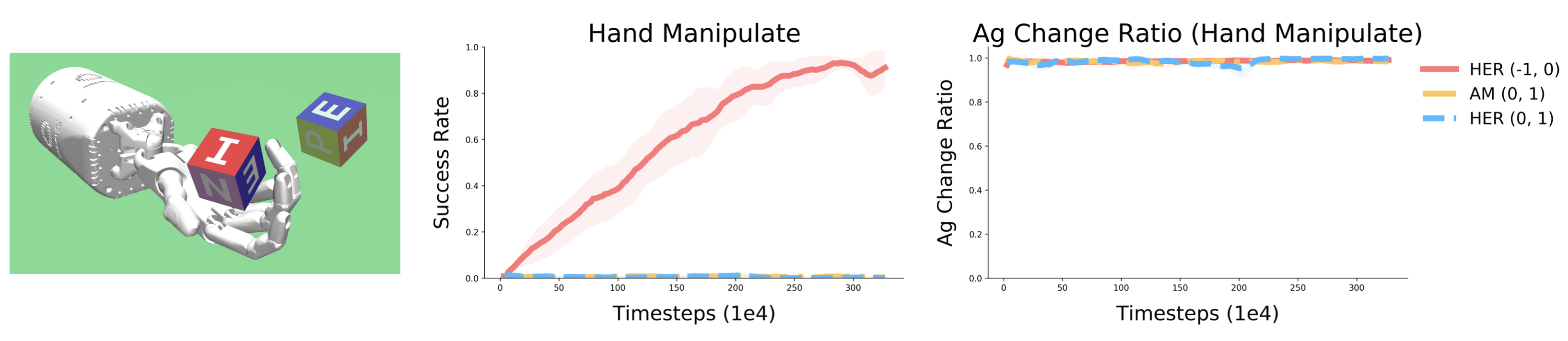}
    \caption{Even on an environment where the \textit{initial ag change ratio} is $1$ and where the \textit{ag change ratio} stays at around $1$ throughout learning for all methods, Actionable Models (AM) \citep{actionablemodels} and HER with $\{0, 1\}$ rewards still fail to learn anything, where HER with $\{-1, 0\}$ rewards \citep{her} succeed at learning a goal-conditioned policy. All hyper-parameters are the same except for the reward design and bellman backups for all three methods. Results are averaged over 5 random seeds, and the variances across seeds are plotted. This shows that our conclusions in Section \ref{sec:reward-design-results} are independent of exploration difficulties, and that the design choices for goal-reaching rewards matter significantly.}
    \label{fig:hand-env-results}
\end{figure}

\end{document}